\documentclass{article} 
\usepackage{dfr_conference,times}

\usepackage{amsmath,amsfonts,bm}

\def\eqref#1{equation~\ref{#1}}

\def\1{\bm{1}}

\def\vc{{\bm{c}}}

\def\vv{{{v}}}

\def\vx{{\bm{x}}}
\def\vy{{\bm{y}}}

\DeclareMathAlphabet{\mathsfit}{\encodingdefault}{\sfdefault}{m}{sl}
\SetMathAlphabet{\mathsfit}{bold}{\encodingdefault}{\sfdefault}{bx}{n}

\DeclareMathOperator*{\argmin}{arg\,min}

\usepackage{hyperref}
\usepackage{url}
\usepackage[most]{tcolorbox}
\usepackage{xcolor}
\usepackage{mathtools}
\usepackage{booktabs}
\usepackage[font=small]{caption}
\usepackage{graphicx}
\usepackage[normalem]{ulem}
\usepackage{amssymb}
\usepackage{enumitem}
\usepackage{amsthm}
\usepackage{thm-restate}
\usepackage{xspace}
\usepackage{algorithm}
\usepackage{algpseudocode}
\usepackage{colortbl}
\usepackage{subcaption}
\newcommand{\mstd}[2]{\begin{tabular}[c]{@{}c@{}}#1\\[-2pt]{\scriptsize$\pm$#2}\end{tabular}}

\newtcolorbox{questionbox}{
    colback=gray!8,
    colframe=gray!8,
    boxrule=0pt,
    arc=2pt,
    left=1pt,
    right=1pt,
    top=1pt,
    bottom=1pt,
    width=\linewidth,
    halign=center,
    before=\begin{center},
    after=\end{center},
}

\definecolor{lightgrayblue}{RGB}{236,238,252}
\definecolor{grayblue}{RGB}{120,135,155}
\definecolor{lightred}{RGB}{253,246,246}
\definecolor{mutedred}{RGB}{175,105,105}

\newcommand{\sg}{\operatorname{sg}}
\newcommand{\FD}{\mathrm{FD}}

\newcommand{\methodname}{\textsc{DFR}\xspace}

\title{Distribution-based regularization for reward fine-tuning diffusion and flow models}

\author{
\begin{tabular}{cccc}
\textbf{Romeo Passaro}$^{*,1,2}$ &
\textbf{Jeongjae Lee}$^{*,3,4,5}$ &
\textbf{Jong Chul Ye}$^{3}$ &
\textbf{Michael M. Bronstein}$^{1,2}$
\end{tabular}
\\[0.8em]
\begin{tabular}{ccc}
\textbf{Yoshua Bengio}$^{4,5}$ &
\textbf{Alexander Tong}$^{2,\dagger}$ &
\textbf{Danyal Rehman}$^{2,4,5,\dagger}$
\end{tabular}
\\[1.0em]
{\small
$^{1}$University of Oxford
\quad
$^{2}$AITHYRA
\quad
$^{3}$KAIST
\quad
$^{4}$Mila – Québec AI Institute
\quad
$^{5}$Université de Montréal
}
\\[0.35em]
{\small
$^{*}$Equal contribution; ordering determined by a random draw.
\qquad
$^{\dagger}$Corresponding authors.
}
}

\iclrfinalcopy
\begin{document}

\maketitle
\lhead{Preprint.}
\begin{abstract}
\looseness=-1
Reward fine-tuning is an effective way to adapt pretrained diffusion and flow-matching models to downstream objectives, but aggressive optimization often leads to mode collapse and distributional drift. Existing methods attempt to control this drift with reverse-KL regularization, implemented as an $\ell_2$ penalty on the velocity field;\ however, this functions by indirectly constraining sampling dynamics rather than directly acting on the generated distribution. To address this, we propose \methodname (Distributional Functional Rewards), which augments the task reward with a distribution-dependent, per-sample reward that acts explicitly on the generated distribution. Formally, we generalize the KL-regularized objective from expected rewards, which are linear in the model's distribution, to convex functionals of it. The first variation of this objective defines an effective, distribution-dependent reward that directly penalizes deviation from the pretrained distribution. This formulation yields additive per-sample rewards for the Fréchet distance (FD), maximum mean discrepancy (MMD), and DreamSim diversity, recovering several recently proposed rewards as special cases. Across various text-to-image fine-tuning tasks, 
\methodname matches the peak reward of KL-tuned baselines while increasing diversity by up to 29\% at equal or lower drift.
\end{abstract}
\begin{center}
\vspace{-0.5em}
\textbf{Code:} \url{https://github.com/jaylee2000/dfr}
\end{center}

\section{Introduction}
\looseness=-1
Diffusion and flow-matching models~\citep{pmlr-v37-sohl-dickstein15,ho2020ddpm,song2021scorebased,lipman2023flow,liu2023flow,albergo_stochastic_2023,peluchetti2023non} have become a dominant paradigm for generative modeling across various data modalities, including images~\citep{rombach2022high,labs2025flux1kontextflowmatching}, videos~\citep{wan2025wanopenadvancedlargescale}, and molecular structures~\citep{alphafold3}.
These models learn stochastic or deterministic dynamics that transport a simple source distribution to a complex target distribution.
However, downstream applications rarely value all samples equally:\ the goal is often to shift probability mass toward high-reward samples, while preserving the diversity and capabilities of the pretrained model.
\emph{Reward fine-tuning} achieves this by adapting the pretrained model toward a downstream reward~\citep{uehara2024understandingreinforcementlearningbasedfinetuning}.

\looseness=-1
For non-differentiable rewards, existing approaches broadly fall into two families:\ (1) \emph{Trajectory}-based methods, which formulate sampling as a multi-step Markov decision process (MDP), treating each step of the sampling process as an action~\citep{black2023ddpo,fan2023dpok};\ and (2) \emph{ELBO}-based methods, which instead treat the final sample as the action and approximate its intractable log-likelihood with an evidence lower bound~\citep{mcallister2026flow}. Recently, ELBO-based methods have emerged as particularly attractive options for large-scale text-to-image fine-tuning, reporting up to $\sim 50\times$
faster training than trajectory-based methods at comparable reward and image quality~\citep{zheng2025diffusionnft,xue2025advantageweightedmatchingaligning,bergmeister2026reinforceadjointmatchingscaling}. 

\looseness=-1
Across both families, however, reward maximization alone does not specify which properties of the pretrained distribution ought to be preserved.
In particular, optimizing toward increasingly high-reward samples can reduce sample diversity and move the generated distribution away from the pretrained model.
Existing methods often address this tension through KL regularization, typically implemented as an $\ell_2$ penalty between the fine-tuned and pretrained velocity fields~\citep{fan2023dpok,liu2025flowgrpo}. However, this regularization acts on the model's \emph{sampling dynamics}, whereas the properties of ultimate interest---diversity and retention of the pretrained distribution---are properties of its \emph{output distribution}.
This raises the natural question: 

\begin{questionbox}
Can we regularize the generated distribution directly during reward fine-tuning?
\end{questionbox}

We answer this question affirmatively with Distributional Functional Rewards (\methodname), a general framework for distribution-regularized reward fine-tuning. \methodname unifies trajectory- and ELBO-based methods by augmenting reward optimization with objectives defined directly on the generated distribution. This allows models to improve reward while better preserving the pretrained distribution. Our main contributions are:
\looseness=-1
\begin{itemize}[leftmargin=*,itemsep=0pt,topsep=2pt]
\looseness=-1
\item \textbf{A unified framework for distribution-regularized reward fine-tuning.}
\methodname generalizes KL-regularized reward optimization to convex functionals of the clean-sample distribution, including maximum mean discrepancy (MMD), Fr\'echet distance (FD), and diversity objectives. Their first variations induce distribution-dependent per-sample reward corrections, enabling direct control of the generated distribution in both trajectory- and ELBO-based methods.
\looseness=-1
\item \textbf{Improved distribution preservation at matched reward.}
On GenEval and OCR with SD3.5-M, \methodname reduces drift from the pretrained model by up to 12\% and increases diversity by up to 29\% relative to KL-regularized baselines at matched reward, while training 36--41$\times$ faster than Flow-GRPO.
\looseness=-1
\item \textbf{Compatibility with both feature- and latent-space objectives.}
We instantiate \methodname with feature-space objectives such as FD and DreamSim diversity, as well as latent-space MMD, demonstrating strong performance even without the use of pretrained feature encoders.
\looseness=-1
\item \textbf{Minimal overhead over KL-regularized fine-tuning.}
Compared with KL-tuned RAM, \methodname adds little training cost: \methodname-MMD reaches GenEval 0.95 with only 7\% more training time than RAM ($\beta=$ 3) while achieving 9\% higher diversity, and \methodname-DS reaches OCR 0.96 with only 1\% more time than RAM ($\beta=$ 5) while achieving 17\% higher diversity.
\end{itemize}

\begin{figure*}[t]
\centering
\includegraphics[width=\linewidth]{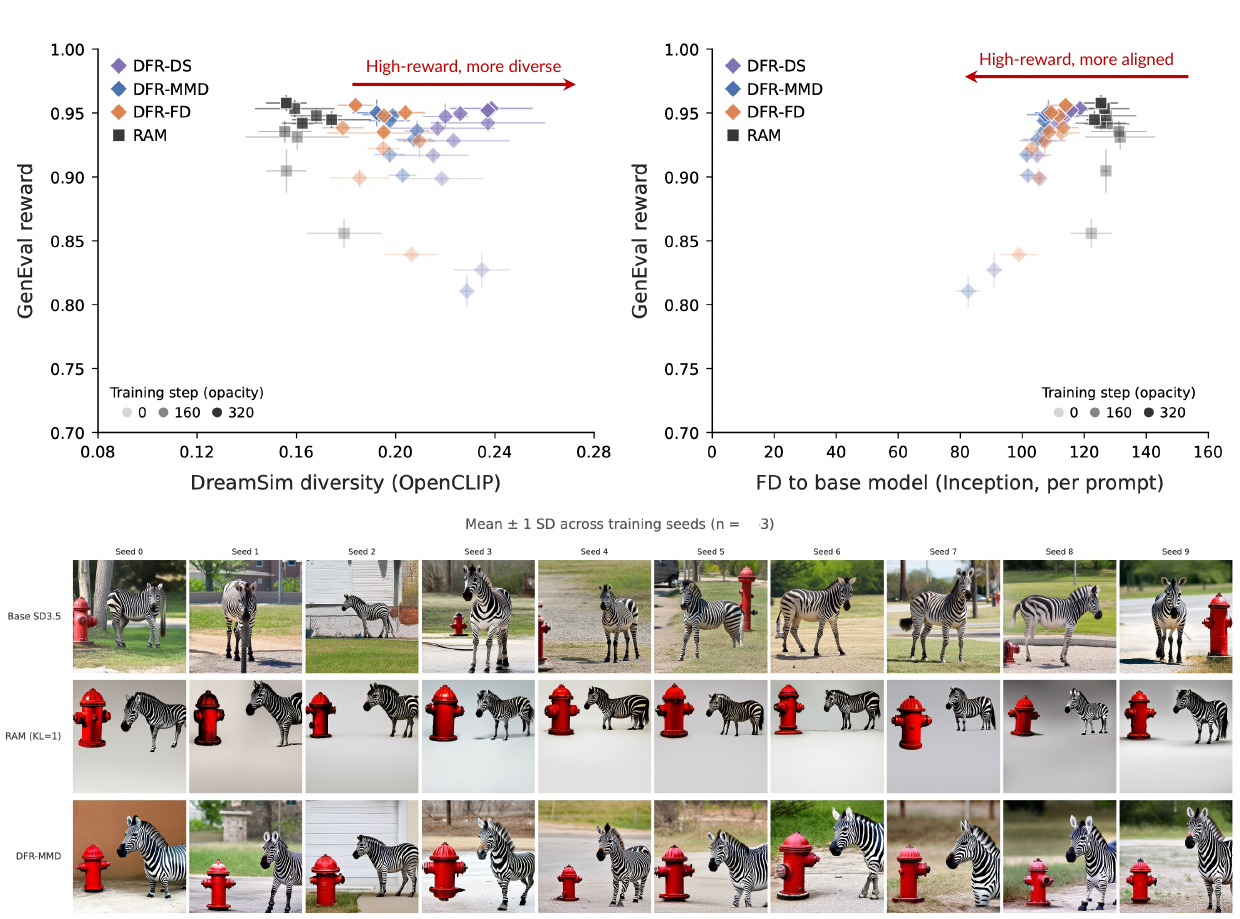}
\caption{\textbf{Top}: RAM achieves high rewards at the cost of sample diversity and distributional fidelity, whereas \methodname matches---and often surpasses---its rewards while preserving diversity and FID relative to the pretrained distribution. \textbf{Bottom}: \methodname more closely mirrors the pretrained SD3.5 distribution, generating a wider diversity of images relative to RAM across multiple seeds for the same prompt: \textit{``A zebra to the right of a fire hydrant.''}}
\label{fig:reward-curves}
\end{figure*}

\section{Background}
\subsection{Flow-Based Models}
\label{sec:flow_preliminaries}
\looseness=-1
Flow-based models---including diffusion and flow-matching frameworks~\citep{lipman2023flow,song2021scorebased}---learn a time-dependent velocity field $\vv_\theta(\vx_t,t)$ that transports samples between a source distribution, typically Gaussian noise $\pmb\epsilon\sim\mathcal N(0,I)$, and the data distribution $\vx_0\sim p_{\mathrm{data}}$. For the commonly used linear path $\vx_t=(1-t)\vx_0+t\pmb\epsilon$, the corresponding target velocity is simply $\dot{\vx}_t=\pmb\epsilon-\vx_0$. The velocity field is therefore learned by weighted regression toward this target:
\begin{equation}
    \mathcal L_{\mathrm{FM}}(\theta)
    =
    \mathbb E_{t,\vx_0,\pmb\epsilon}
    \left[
        w(t)
        \left\|
        \vv_\theta(\vx_t,t)-(\pmb\epsilon-\vx_0)
        \right\|_2^2
    \right],
    \label{eq:flow_matching}
\end{equation}
where $w(t)>0$ is a timestep-dependent weighting function. At inference, samples are generated by drawing $\pmb\epsilon\sim\mathcal N(0,I)$ and integrating the learned dynamics backward from $t=1$ to $t=0$. We denote the frozen pretrained velocity field by $\vv_{\text{ref}}$ and the endpoint distribution induced by $\vv_\theta$ under the chosen sampler by $p_\theta(\vx_0)$, omitting conditioning for brevity. This integration takes the form of an SDE for diffusion models and an ODE for flow-matching models;\ we refer the reader to~\cite{tong_simulation-free_2023,liu2025flowgrpo} for a more detailed equivalence between the two formulations.

\subsection{Reward fine-tuning via policy gradients}
\label{sec:background_pg}
\looseness=-1
Reward fine-tuning adapts a pretrained generative model to a downstream reward while regularizing deviations from the pretrained distribution.\footnote{See~\cite{uehara2025inferencetimealignmentdiffusionmodels} for a review of inference-time reward alignment.}
This is most often formulated as KL-regularized reward maximization:
\begin{equation}
\max_\theta \;
\mathbb E_{\vx_0 \sim p_\theta}[r(\vx_0)]
- \beta \mathcal D_{\mathrm{KL}}
\left(p_\theta(\vx_0)\,\|\,p_{\text{ref}}(\vx_0)\right),
\label{eq:maxentrl}
\end{equation}
\looseness=-1
where $\beta>0$ controls the reverse-KL regularization strength, and $p_\theta$ and $p_{\mathrm{ref}}$ are the endpoint distributions of the fine-tuned and pretrained models, respectively. For non-differentiable rewards, the gradient of the expected rewards can be evaluated using the REINFORCE identity~\citep{Williams:92}:
\begin{equation}
\nabla_\theta \mathbb E_{\vx_0\sim p_\theta}[r(\vx_0)]
=
\mathbb E_{\vx_0\sim p_\theta}
\left[
(r(\vx_0)-b)\nabla_\theta\log p_\theta(\vx_0)
\right],
\label{eq:reinforce}
\end{equation}
where $b$ is a sample-independent baseline used to reduce variance, and the difference $r(\vx_0)-b$ defines the advantage. In practice, advantages are often normalized within groups of samples from the same prompt~\citep{shao2024deepseekmath}, while clipped likelihood ratios can further stabilize training optimization~\citep{schulman2017proximalpolicyoptimizationalgorithms}.

For flow-based models, however, the marginal log-likelihood $\log p_\theta(\vx_0)$ is intractable. Trajectory-based methods replace it with the per-step transition likelihoods of a stochastic sampler.
ELBO-based methods instead replace it with $-\mathcal L_\text{FM}(\theta;\vx_0)$, the loss of Eq.~(\ref{eq:flow_matching}) evaluated on a single generated sample.
For suitable $w(t)$, this is an evidence lower bound up to an additive constant~\citep{mcallister2026flow,choi2026rethinking}.
Substituting it into Eq.~(\ref{eq:reinforce}) gives online reward-weighted regression
(RWR): the model is trained with the flow-matching loss on its own samples, each
weighted by its advantage~\citep{fan2025online,xue2025advantageweightedmatchingaligning}.
We build on RAM~\citep{bergmeister2026reinforceadjointmatchingscaling}, which takes this
form.
See Appendix~\ref{appendix:unifying-online-rwr} for a detailed derivation of these equivalences.

\textbf{Regularization.}
While trajectory- and ELBO-based methods differ in how they estimate the reward gradient, both typically approximate the reverse-KL term in Eq.~(\ref{eq:maxentrl}) using a squared velocity anchor:
\begin{equation}
    \mathcal L_{\text{anchor}}(\theta)
    =
    \mathbb E
    \left[
        \left\|
            \vv_\theta(\vx_t,t)-\vv_{\text{ref}}(\vx_t,t)
        \right\|_2^2
    \right].
    \label{eq:velocity_anchor}
\end{equation}
For Gaussian transition kernels, a time-weighted form of Eq.~(\ref{eq:velocity_anchor}) equals the KL divergence between the path distributions
of the fine-tuned and pretrained samplers, which upper-bounds the reverse-KL divergence of Eq.~(\ref{eq:maxentrl})~\citep[Equation~7]{fan2023dpok}.

\subsection{Distributional Measures}
\label{sec:distributional_objectives}
\looseness=-1
\textbf{Fr\'echet distance (FD).}
For a feature encoder $\phi$, let $\mu_p$ and $\Sigma_p$ denote the mean and covariance of $\phi(\vx)$ under $\vx \sim p$. The Fr\'echet distance (FD) compares the first two feature moments of $p$ and $q$:
\begin{equation}
\FD(p,q)
\triangleq
\|\mu_p-\mu_q\|_2^2
+
\operatorname{tr}\!\left(
\Sigma_p+\Sigma_q
-
2\bigl(\Sigma_q^{1/2}\Sigma_p\Sigma_q^{1/2}\bigr)^{1/2}
\right).
\label{eq:fd}
\end{equation}

\textbf{Maximum mean discrepancy (MMD)} is a classical kernel two-sample discrepancy~\citep{mmd2006}. For a chosen kernel $k(\cdot, \cdot)$,
\begin{equation}
\mathrm{MMD}_k^2(p, q)
=
\mathbb E_{\vx,\vx'\sim p}[k(\vx,\vx')]
+
\mathbb E_{\vy,\vy'\sim q}[k(\vy,\vy')]
-
2\mathbb E_{\substack{\vx\sim p\\\vy\sim q}}[k(\vx,\vy)],
\label{eq:mmd}
\end{equation}
where a characteristic kernel ensures $\mathrm{MMD}_k(p,q)=0$ if and only if $p=q$. Note, $\mathrm{MMD}$ requires no feature encoder and can be computed directly in latent space.

\textbf{DreamSim diversity.}
DreamSim~\citep{fu2023dreamsim} is a perceptual distance trained to align with human judgments of visual similarity.
Given the normalized DreamSim feature
$z(\vx) = \phi_{\mathrm{DS}}(\vx) / \|\phi_{\mathrm{DS}}(\vx)\|_2$,
DreamSim diversity~\citep{domingo-enrich2025adjoint} is the mean pairwise cosine
distance between independent samples, which equals the total variance of the
normalized features:
\begin{equation}
\mathrm{Div}_{\mathrm{DS}}(p)
=
\mathbb E_{\vx,\vx'\sim p}\bigl[1 - z(\vx)^\top z(\vx')\bigr]
=
1 - \bigl\|\mathbb E_{\vx\sim p}[z(\vx)]\bigr\|_2^2
=
\operatorname{tr}\,\mathrm{Cov}_{\vx\sim p}[z(\vx)].
\label{eq:dreamsim-diversity}
\end{equation}

\section{\methodname: Distributional Regularization for Reward Fine-Tuning}
\looseness=-1
The velocity anchor in Eq.~(\ref{eq:velocity_anchor}) only indirectly controls the reverse-KL divergence in Eq.~(\ref{eq:maxentrl}), as it penalizes changes in the sampling dynamics even when the endpoint distribution $p_\theta$ is unchanged. In ELBO-based methods, heuristic time weighting at forward-noised states further weakens this correspondence. We nevertheless retain the velocity anchor for training stability and to limit drift, yet we complement it with a term that acts directly on the generated distribution. Importantly, the two objectives are compatible:\ the optimal fine-tuned process preserves the pretrained model's conditional noising process (demonstrated in Appendix~\ref{app:backward-consistency}).

We generalize the KL-regularized objective of Eq.~(\ref{eq:maxentrl}) from expected rewards to convex functionals of the generated distribution
(Section~\ref{sec:distributional_regularization}). The first variation of such a functional acts as a per-sample reward, which any
policy gradient method can optimize without further changes (Section~\ref{sec:method_online_rwr}), and instantiate the framework with MMD, FD, and DreamSim diversity (Section~\ref{subsec:special-cases-of-ours}).

\subsection{Distributional regularization through convex functionals}
\label{sec:distributional_regularization}
Let \(\mathcal F:\mathcal P(\mathcal X)\rightarrow\mathbb R\) be a convex functional of the generated distribution.
Consider the entropy-regularized optimization problem
\begin{equation}
p^\star
=
\argmin_{p\in\mathcal P(\mathcal X)}
\left(
\mathcal F(p)
+
\beta\,\mathcal{D}_\text{KL}(p\Vert p_\text{ref})
\right),
\qquad \beta>0.
\label{eq:general_functional_problem}
\end{equation}
For a signed measure $\eta$ on $\mathcal X$ with $\eta(\mathcal X)=0$, the first-order
change of $\mathcal F$ along $\eta$ defines the \emph{functional derivative}
$\frac{\delta\mathcal F}{\delta p}(p,\cdot):\mathcal X\rightarrow\mathbb R$ through
\begin{equation}
\lim_{\varepsilon\to 0}
\frac{\mathcal F(p+\varepsilon\eta)-\mathcal F(p)}{\varepsilon}
\triangleq
\int \frac{\delta\mathcal F}{\delta p}(p,\vx)\,\eta(\mathrm d\vx).
\label{eq:functional_derivative}
\end{equation}
Since $\eta$ integrates to zero, $\frac{\delta\mathcal F}{\delta p}$ is defined only
up to an additive constant. Under standard regularity and integrability assumptions, this (strictly) convex minimization problem has a unique minimum. Standard variational calculus techniques identify the unique minimizer with the following equation:
\begin{equation}
p^\star(\vx) = \frac{1}{Z(p^\star)}\, p_\text{ref}(\vx)
\exp\left(
-\frac{1}{\beta}
\frac{\delta\mathcal F}{\delta p}(p^\star,\vx)
\right),
\label{eq:general_functional_gibbs}
\end{equation}
where \(Z(p^\star)\) is the normalizing constant.
Eq.~(\ref{eq:general_functional_gibbs}) is generally an implicit fixed point condition because the functional derivative may itself depend on \(p^\star\).
We refer the reader to \citet{hu2021meanfield} for a formal treatment of this problem, and to Appendix~\ref{sec:loose-derivation} for a looser derivation.
We apply this result to two choices of $\mathcal F$: the expected reward, and the
expected reward with a convex distributional penalty.

\textbf{Expected rewards as the linear case.}
Standard expected reward maximization, Eq.~(\ref{eq:maxentrl}), corresponds to $\mathcal F(p)=-\mathbb E_{\vx\sim p}[r(\vx)]$, which is linear in $p$ with first variation $-r(\vx)$. Eq.~(\ref{eq:general_functional_gibbs}) then reduces to the familiar exponential tilt:
\begin{equation}
p^\star(\vx)
=
\frac{1}{Z}\,
p_\text{ref}(\vx)\exp\left(\frac{r(\vx)}{\beta}\right),
\label{eq:standard_reward_tilt}
\end{equation}
whose right-hand side has no dependence on $p^\star$.

\textbf{Distribution-level effective reward.}
Distributionally regularized reward maximization corresponds to
\begin{equation}
\mathcal F(p) = -\mathbb E_{\vx\sim p}[r(\vx)] + \lambda\,\mathcal R(p),
\qquad \lambda\geq 0,
\label{eq:general_distributional_regularization}
\end{equation}
where $\mathcal R(p)$ is a convex functional that measures how far $p$ is from
$p_\text{ref}$. For brevity, we omit its dependence on the fixed $p_\text{ref}$.
This $\mathcal F$ is no longer linear but remains convex in $p$, with first variation
\begin{equation}
\frac{\delta\mathcal F}{\delta p}(p,\vx)
= -r(\vx) + \lambda\,\frac{\delta\mathcal R}{\delta p}(p,\vx),
\end{equation}
which, unlike in the linear case, depends on $p$.
We call its negative the \emph{effective reward}\footnote{Here and below, the argument after the semicolon is the distribution at which the
derivative is taken, and $\vx$ is the sample at which it is evaluated. For example,
$\mathcal R(p)=\tfrac12\|\mathbb E_{\vx\sim p}[z(\vx)]\|_2^2$ gives
$r_{\mathcal R}(\vx;p)=-z(\vx)^\top\mathbb E_{\vx'\sim p}[z(\vx')]$.}:
\begin{equation}
r_{\text{eff}}(\vx;p) \triangleq r(\vx) + \lambda\, r_{\mathcal R}(\vx;p),
\qquad
r_{\mathcal R}(\vx;p) \triangleq -\frac{\delta\mathcal R}{\delta p}(p,\vx).
\label{eq:general_effective_reward}
\end{equation}

Substituting Eq.~(\ref{eq:general_effective_reward}) into
Eq.~(\ref{eq:general_functional_gibbs}) gives
\begin{equation}
p^\star(\vx) = \frac{1}{Z(p^\star)}\,p_\text{ref}(\vx)
\exp\left(\frac{r_{\text{eff}}(\vx;p^\star)}{\beta}\right).
\label{eq:general_distribution_dependent_tilt}
\end{equation}
This is the exponential tilt of Eq.~(\ref{eq:standard_reward_tilt}) with $r$ replaced by
$r_{\text{eff}}$, so the penalty acts exactly like an additional reward. Unlike
Eq.~(\ref{eq:standard_reward_tilt}), it cannot be evaluated directly, because
$r_{\text{eff}}$ depends on the unknown $p^\star$. We therefore approach $p^\star$ by
gradient descent, evaluating $r_{\text{eff}}$ at the current model instead.

\subsection{Optimizing distributional functionals via policy gradients}
\label{sec:method_online_rwr}

\begin{algorithm}[!t]
\caption{Online RWR with \methodname. The shaded line is the only change to the baseline.}
\label{alg:method}
\begin{algorithmic}[1]
\Require {pretrained velocity $\vv_\text{ref}$; target reward $r$; distributional functional $\mathcal R$; coefficients $\beta$, $1/\alpha$, $\lambda$; group size $G$}
\State $\theta \gets \theta_\text{ref}$, \quad $\theta_\text{old} \gets \theta_\text{ref}$
\For{each iteration}
    \For{each prompt $\vc$}
        \State Sample a group $\{\vx_0^{i}\}_{i=1}^{G} \sim p_{\theta_\text{old}}(\cdot \mid \vc)$
        \State $A^{i} \gets \big(r(\vx_0^{i}, \vc) - \operatorname{mean}_j\, r(\vx_0^{j}, \vc)\big) / \sigma_r$
            \Comment{group-centered advantage}
        \State \colorbox{lightgrayblue}{$A^{i} \gets A^{i} + \lambda \big(r_{\mathcal R}(\vx_0^{i}; \hat p) - \operatorname{mean}_j\, r_{\mathcal R}(\vx_0^{j}; \hat p)\big)$}
            \Comment{$r_{\mathcal R} = -\delta \mathcal R / \delta p$}
    \EndFor
    \For{each sample $(\vx_0, \vc, A)$, with $t \sim p(t)$ and $\pmb\epsilon \sim \mathcal N(0, I)$}
        \State $\vx_t \gets (1-t)\,\vx_0 + t\,\pmb\epsilon$, \quad $\vv_\text{fwd} \gets \pmb\epsilon - \vx_0$
        \State $\vv_\text{tgt} \gets \sg\!\big[\vv_\theta + \beta\,(\vv_\text{ref} - \vv_\theta) + \tfrac{1}{\alpha} A\,(\vv_\text{fwd} - \vv_\text{old})\big]$
            \Comment{all velocities at $(\vx_t, t, \vc)$}
        \State Take a gradient step on $\|\vv_\theta(\vx_t, t, \vc) - \vv_\text{tgt}\|_2^2$
    \EndFor
    \State Update $\theta_\text{old}$ as a lagged (EMA) copy of $\theta$
\EndFor
\end{algorithmic}
\end{algorithm}

We minimize $\mathcal F(p_\theta)$ over the generator parameters $\theta$; the reverse-KL
term of Eq.~(\ref{eq:general_functional_problem}) is handled separately below. An
infinitesimal change in $\theta$ perturbs the distribution in the direction
$\eta=\nabla_\theta p_\theta$, which has zero total mass, so the definition in
Eq.~(\ref{eq:functional_derivative}) gives the chain rule
\begin{equation}
\nabla_\theta\mathcal F(p_\theta)
= \int \frac{\delta\mathcal F}{\delta p}(p_\theta,\vx)\,
\nabla_\theta p_\theta(\vx)\,\mathrm d\vx .
\label{eq:functional_score_identity}
\end{equation}
Using $\nabla_\theta p_\theta=p_\theta\nabla_\theta\log p_\theta$ and
Eq.~(\ref{eq:general_effective_reward}),
\begin{equation}
-\nabla_\theta\mathcal F(p_\theta)
= \mathbb E_{\vx\sim p_\theta}\left[
r_{\text{eff}}(\vx;p_\theta)\,\nabla_\theta\log p_\theta(\vx)\right].
\label{eq:effective_policy_gradient}
\end{equation}
This is the REINFORCE identity of Eq.~(\ref{eq:reinforce}) with $r_{\text{eff}}$ in place
of $r$. As in the ordinary chain rule, the derivative is evaluated at the current
$p_\theta$ and is not itself differentiated: $r_{\text{eff}}(\vx;p_\theta)$ is a scalar
weight attached to each sample, and no gradient flows through $\mathcal R$ or any feature
extractor it uses.

In practice, we draw a group of samples from $p_\theta$, estimate
$r_{\mathcal R}(\vx^i;p_\theta)$ for each sample from the statistics of that group
(denoted $\hat p$), and add $\lambda\,r_{\mathcal R}$ to the sample's advantage
(Algorithm~\ref{alg:method}, shaded line). Since
$\mathbb E_{\vx\sim p_\theta}[\nabla_\theta\log p_\theta(\vx)]=0$, adding a constant to
$r_{\text{eff}}$ leaves the gradient unchanged; the additive constant left free in
Eq.~(\ref{eq:functional_derivative}) is the baseline $b$ of Eq.~(\ref{eq:reinforce}),
which is why we may center $r_{\mathcal R}$ within each group.

\textbf{Application to flow models.}
\looseness=-1
We use RAM~\citep{bergmeister2026reinforceadjointmatchingscaling}, the state-of-the-art ELBO-based method,
as our base framework. The score in Eq.~(\ref{eq:effective_policy_gradient}) is replaced by the ELBO surrogate
$-\nabla_\theta\mathcal L_{\mathrm{FM}}(\theta;\vx)$ alluded to in~Eq.~(\ref{eq:flow_matching}), and the reverse-KL term in Eq.~(\ref{eq:general_functional_problem}) by $\mathcal L_\text{anchor}(\theta)$ (Eq.~(\ref{eq:velocity_anchor})). The complete training pipeline is provided in Algorithm~\ref{alg:method}.

\looseness=-1
Algorithm~\ref{alg:method} implements the population formulation with three finite-sample choices. All quantities are conditional:\ $\mathcal R$ is applied to $p_\theta(\cdot\mid\vc)$ and 
averaged over conditions $\vc$, which we omit for brevity. 
The distributional reward is centered within each group and added to the normalized task advantage rather than to $r$ before normalization;\ the two differ only in the rescaling of $\lambda$ by the advantage scale. The population statistics entering $r_{\mathcal R}$ are estimated from the group of $G$ samples for MMD and DreamSim diversity, excluding the sample itself from pairwise averages, and from an exponential moving average of feature moments pooled over prompts for FD (detailed in Appendix~\ref{app:fd-numerics}).

\subsection{Instantiations: MMD, FD, Dreamsim Diversity}
\label{subsec:special-cases-of-ours}
\looseness=-1
We instantiate $\mathcal R$ with three functionals: (1) the MMD reward, our primary choice, which requires no pretrained encoder; (2) the FD reward; and (3) the DreamSim reward.
The latter two show that \methodname subsumes recently proposed diversity and FD-based rewards as special cases; Section~\ref{sec:experiments} evaluates all three under the same protocol.
All three functionals are convex in $p$ (Appendix~\ref{app:convexity}), so the analysis in Section~\ref{sec:distributional_regularization} applies directly (for FD, see the remark in Appendix~\ref{app:fd-numerics}).

\textbf{MMD reward.} Let $\psi:\mathcal X\to\mathbb R^m$ be a fixed injective map and $b\in(0,2)$, and define
\begin{equation}
\rho_b(\vx,\vy)=\|\psi(\vx)-\psi(\vy)\|_2^b.
\label{eq:energy-semimetric}
\end{equation}
We use the identity map for $\psi$, meaning $\rho_b$ reduces to a power of the Euclidean distance in latent space and requires no pretrained encoder, although an encoder could be used instead. The corresponding MMD can be written as
\begin{align}
\operatorname{MMD}_{k_b}^2(p,p_\text{ref})
&=
\mathbb E_{\substack{\vx\sim p\\\vy\sim p_\text{ref}}}
\left[
\rho_b(\vx,\vy)
\right]
-
\frac{1}{2}
\mathbb E_{\vx,\vx'\sim p}
\left[
\rho_b(\vx,\vx')
\right]
-
\frac{1}{2}
\mathbb E_{\vy,\vy'\sim p_\text{ref}}
\left[
\rho_b(\vy,\vy')
\right].
\label{eq:energy-mmd-functional}
\end{align}
The proof is in Appendix~\ref{app:energy-mmd}. Consider the MMD functional 
$\mathcal R_{\mathrm{MMD}}(p,p_\text{ref})=\operatorname{MMD}_{k_b}^2(p,p_\text{ref})$.
\begin{restatable}[MMD reward]{proposition}{MMDEffectiveReward}
\label{prop:mmd-effective-reward}
A representative of the negative functional derivative of
\(\mathcal R_{\mathrm{MMD}}\) is
\begin{equation}
\boxed{
r_{\mathrm{MMD}}(\vx;p,p_\text{ref})
=
\mathbb E_{\vx'\sim p}
\left[
\rho_b(\vx,\vx')
\right]
-
\mathbb E_{\vy\sim p_\text{ref}}
\left[
\rho_b(\vx,\vy)
\right].
}
\label{eq:mmd-effective-reward}
\end{equation}
\end{restatable}

The proof is given in Appendix~\ref{app:mmd-reward}.
The two terms in Eq.~(\ref{eq:mmd-effective-reward}) have complementary roles.
The first is a repulsive term that rewards samples far from other samples
generated by the current model, thereby discouraging collapse. The second is an
attractive term that penalizes samples far from the pretrained distribution,
thereby limiting distributional drift.

\textbf{FD reward.}
\methodname also recovers FD-based reward augmentation.
With $\mu_p,\Sigma_p$ as in Section~\ref{sec:distributional_objectives}, consider the FD functional $\mathcal R_{\mathrm{FD}}(p)=\FD(p,p_\text{ref})$ of Eq.~(\ref{eq:fd}).

\begin{restatable}[FD reward]{proposition}{FDEffectiveReward}
\label{prop:fd-effective-reward}
Assume the covariance matrices are positive definite and define
\begin{equation}
G_p
=
I
-
\Sigma_q^{1/2}
\left(
\Sigma_q^{1/2}\Sigma_p\Sigma_q^{1/2}
\right)^{-1/2}
\Sigma_q^{1/2}
\label{eq:fd-G}
\end{equation}
for $q \equiv p_\text{ref}$.
A representative of the negative functional derivative of
$\mathcal R_{\mathrm{FD}}$ is
\begin{equation}
\boxed{
r_{\mathrm{FD}}(\vx;p)
=
-
\left[
2(\mu_p-\mu_q)^\top f
+
(f-\mu_p)^\top G_p(f-\mu_p)
\right].
}
\label{eq:fd-effective-reward}
\end{equation}
\end{restatable}

The proof is given in Appendix~\ref{app:fd-reward}. The two terms in Eq.~(\ref{eq:fd-effective-reward}) have complementary roles.
The first penalizes shifts in the population feature mean, while the second corrects
changes in feature covariance.

\textbf{DreamSim diversity reward.}
\methodname also recovers pairwise perceptual diversity rewards.
With $z(\vx)$ defined in Section~\ref{sec:distributional_objectives},  the cosine similarity between two samples is
$z(\vx)^\top z(\vx')$.
We consider the population-average cosine-similarity functional
$\mathcal R_{\mathrm{DS}}(p)
=
\frac12 (1-\mathrm{Div}_{\mathrm{DS}}(p))$
of Eq.~(\ref{eq:dreamsim-diversity}).

\begin{restatable}[DreamSim reward]{proposition}{DreamSimEffectiveReward}
\label{prop:dreamsim-effective-reward}
A representative of the negative functional derivative of
\(\mathcal R_{\mathrm{DS}}\) is
\begin{equation}
\boxed{
r_{\mathrm{DS}}(\vx;p)
=
-
z(\vx)^\top
\mathbb E_{\vx'\sim p}[z(\vx')].
}
\label{eq:dreamsim-effective-reward}
\end{equation}
\end{restatable}

The proof is given in Appendix~\ref{app:dreamsim-reward}.

\section{Experiments}
\label{sec:experiments}
\looseness=-1
We evaluate whether \methodname improves the trade-off between reward and reward hacking beyond what tuning the reverse-KL coefficient of the baseline can achieve.

\subsection{Setup}
\label{subsec:experimental-setup}
\textbf{Tasks and methods.}
\looseness=-1
We fine-tune Stable Diffusion 3.5 Medium (SD3.5-M;~\citealp{esser2024sd3}) on two tasks: (1) Compositional generation (GenEval;~\citealp{ghosh2023geneval}), with reward thresholds 0.95 and 0.96, and (2) Visual text rendering (OCR), with threshold 0.96.
We adopt RAM~\citep{bergmeister2026reinforceadjointmatchingscaling} as the ELBO-based framework on which \methodname builds on.
We sweep its reverse-KL coefficient $\beta$ over $\{1,\dots,4\}$ for GenEval and $\{2,\dots,6\}$ for OCR.
\methodname-MMD, \methodname-FD, and \methodname-DS augment RAM with the distributional rewards introduced in Section~\ref{subsec:special-cases-of-ours} (Algorithm~\ref{alg:method}). For each variant, we use the largest $\beta$ for which RAM reaches the highest reward threshold, or the next smaller value when the distributional reward itself provides substantial regularization toward the pretrained model. All configurations are trained with three random seeds under identical compute budgets. 

\looseness=-1
\textbf{Baselines. } We compare \methodname against other ELBO- and trajectory-based methods including DiffusionNFT~\citep{zheng2025diffusionnft}, AWM~\citep{xue2025advantageweightedmatchingaligning}, RSM~\citep{lee2026rewardscorematchingunifying}, and Flow-GRPO~\citep{liu2025flowgrpo} on both the compositional generation and visual text rendering tasks.

\looseness=-1
\textbf{Protocol.} Because reward hacking intensifies over the course of optimization, we compare models at matched target reward levels.
For each seed, we evaluate the first checkpoint that reaches a given threshold;\ a configuration is considered to reach that threshold only if all three seeds do so.

\looseness=-1
\textbf{Metrics.} Following \citet{liu2025resnabladb,liu2025vggflow,domingo-enrich2025adjoint}, we measure reward hacking by drift and diversity, each computed per-prompt over 64 images. \emph{Drift} is the Fr\'echet distance to the 64 images of $p_\text{ref}$ in Inception-v3 and DINOv2 features. \emph{Diversity} is measured by DreamSim diversity, computed with its DINO and OpenCLIP branches. Preference rewards such as ImageReward~\citep{xu2023imagereward} and PickScore~\citep{kirstain2023pickapic} are reported but not used, since they score each image in isolation and cannot detect a model that generates near-identical images for every seed. Two configurations are tied on a metric if their means differ by at most the larger of their standard deviations;\ otherwise, the configuration with the better mean wins. Entries computed with an encoder that a \methodname instance uses in its reward, or a closely related one, are counted as neither wins nor losses. Consequently, the main text reports the OpenCLIP branch, which no instance trains on. Additional training, evaluation, and baseline details are provided in Appendix~\ref{app:experiment-details}.

\subsection{Results}
\label{subsec:comparison}

\begin{table*}[t]
\centering
\vspace{-1em}
\caption{Fine-tuning on GenEval, compared at matched GenEval (Section~\ref{subsec:experimental-setup}).
Peak is the highest threshold it reaches or, if none, its mean GenEval at the evaluated checkpoints; ``--'' marks a threshold not reached on every seed.
Speedup is the time to reach $0.95$ relative to Flow-GRPO.
Entries of the form $a/b$ report results on the GenEval test prompts ($a$) and on DrawBench ($b$); RAM and \methodname{} entries are means over three seeds with SD beneath, and existing methods are single checkpoints evaluated at $0.95$. Blue marks a \methodname{} entry that wins against every RAM configuration reaching the same threshold; struck-through entries are computed with an encoder related to that instance's reward and are not compared.
}
\label{tab:comparison}
\resizebox{\textwidth}{!}{
\begin{tabular}{lcccccccc}
\toprule
& & & \multicolumn{3}{c}{\textbf{GenEval $0.95$}} & \multicolumn{3}{c}{\textbf{GenEval $0.96$}} \\
\cmidrule(lr){4-6}\cmidrule(lr){7-9}
\textbf{Method} & \textbf{Peak} & \textbf{Speedup} $(\uparrow)$ & \textbf{FID} $(\downarrow)$ & \textbf{FD-DINOv2} $(\downarrow)$ & \textbf{DreamSim} $(\uparrow)$ & \textbf{FID} $(\downarrow)$ & \textbf{FD-DINOv2} $(\downarrow)$ & \textbf{DreamSim} $(\uparrow)$ \\
\midrule
SD3.5-M & 0.63 & -- & -- & -- & 0.302/0.274 &  &  &  \\
\midrule
\multicolumn{9}{l}{\emph{Trajectory-based}} \\
Flow-GRPO &  & 1.0 & 114.32/125.61 & 1019/997 & 0.155/0.189 &  &  &  \\
RSM$^\S$ &  & 6.25 & 132.82/154.29 & 1271/1310 & 0.142/0.183 &  &  &  \\
\midrule
\multicolumn{9}{l}{\emph{ELBO-based}} \\
AWM$^{*}$ &  & 8.0 & 158.88/183.71 & 1412/1578 & 0.095/0.132 &  &  &  \\
DiffusionNFT &  & 24.0 & 144.36/159.22 & 1352/1437 & 0.109/0.143 &  &  &  \\
DiffusionNFT (multi)$^\ddagger$ &  & 2.28 & 145.06/175.20 & 1468/1702 & 0.102/0.106 &  &  &  \\
RAM ($\beta=1$) & 0.95 & 52.3 & \mstd{132.93/144.77}{16.19/16.96} & \mstd{1211/1236}{116/185} & \mstd{0.163/0.193}{0.021/0.021} & -- & -- & -- \\
\midrule
\multicolumn{9}{l}{\emph{RAM with larger $\beta$}} \\
RAM ($\beta=2$) & 0.96 & \mstd{57.2}{6.1} & \mstd{116.99/122.78}{1.69/1.02} & \mstd{1078/1002}{34/11} & \mstd{0.170/0.208}{0.009/0.006} & \mstd{119.87/125.99}{2.03/3.89} & \mstd{1066/1023}{31/33} & \mstd{0.174/0.213}{0.006/0.006} \\
RAM ($\beta=3$) & 0.95 & \mstd{43.9}{4.9} & \mstd{113.19/122.06}{2.83/1.84} & \mstd{1030/991}{11/9} & \mstd{0.184/0.214}{0.006/0.010} & -- & -- & -- \\
RAM ($\beta=4$) & 0.94 & -- & -- & -- & -- & -- & -- & -- \\
\midrule
\multicolumn{9}{l}{\emph{Ours}} \\
\methodname-MMD ($\beta=2$) & 0.96 & \mstd{41.1}{9.1} & \mstd{111.20/120.64}{4.10/6.17} & \mstd{\textcolor{blue}{990}/967}{27/52} & \mstd{\textcolor{blue}{0.200}/0.224}{0.015/0.011} & \mstd{\textcolor{blue}{110.75}/121.67}{3.04/6.05} & \mstd{\textcolor{blue}{992}/977}{5/47} & \mstd{\textcolor{blue}{0.196}/0.221}{0.011/0.009} \\
\methodname-FD ($\beta=2$) & 0.96 & \mstd{35.7}{4.6} & \mstd{111.02/\textcolor{blue}{117.09}}{2.00/4.88} & \mstd{\textcolor{blue}{980}/\textcolor{blue}{932}}{20/46} & \mstd{\textcolor{blue}{0.199}/\textcolor{blue}{0.231}}{0.003/0.006} & \mstd{\textcolor{blue}{111.95}/\textcolor{blue}{120.40}}{4.11/4.05} & \mstd{\textcolor{blue}{994}/\textcolor{blue}{969}}{46/38} & \mstd{\textcolor{blue}{0.193}/0.223}{0.011/0.013} \\
\methodname-DS ($\beta=2$) & 0.96 & \mstd{40.4}{5.2} & \mstd{114.60/121.86}{1.64/4.22} & \mstd{\sout{1034}/\sout{992}}{\sout{37}/\sout{82}} & \mstd{\textcolor{blue}{0.218}/\textcolor{blue}{0.233}}{0.007/0.010} & \mstd{\textcolor{blue}{114.47}/124.11}{0.31/3.88} & \mstd{\sout{1052}/\sout{1007}}{\sout{8}/\sout{45}} & \mstd{\textcolor{blue}{0.225}/\textcolor{blue}{0.240}}{0.022/0.018} \\
\bottomrule
\end{tabular}
}
\smallskip
\begin{minipage}{\linewidth}
\scriptsize
$^\S$RSM jumps from 0.94 to 0.97 between consecutive checkpoints, so we evaluate it at 0.94. \\
$^*$AWM's default configuration uses a $1.5\times$ larger effective batch size per step; we match the batch size to the other methods. \\
$^\ddagger$Released checkpoint (GenEval 0.94), sampled at its prescribed guidance scale of 1.0.
\end{minipage}
\vspace{-1em}
\end{table*}

\begin{table*}[t]
\centering
\caption{Fine-tuning on OCR, compared at OCR reward $0.96$. Columns and markings are as in Table~\ref{tab:comparison}, except that entries report results on the OCR test prompts ($a$) and on DrawBench ($b$), and Time is the time to reach $0.96$ relative to RAM ($\beta=2$).}
\label{tab:beta-sweep}
\resizebox{0.8\textwidth}{!}{
\begin{tabular}{lccccc}
\toprule
& & & \multicolumn{3}{c}{\textbf{OCR $0.96$}} \\
\cmidrule(lr){4-6}
\textbf{Method} & \textbf{Peak} & \textbf{Time} $(\downarrow)$ & \textbf{FID} $(\downarrow)$ & \textbf{FD-DINOv2} $(\downarrow)$ & \textbf{DreamSim} $(\uparrow)$ \\
\midrule
SD3.5-M & 0.59 & -- & -- & -- & 0.258/0.274 \\
\midrule
RAM ($\beta=2$) & 0.96 & 1.00 & \mstd{219.07/131.38}{24.04/7.71} & \mstd{1270/1040}{232/84} & \mstd{0.128/0.245}{0.010/0.009} \\
RAM ($\beta=3$) & 0.95 & -- & -- & -- & -- \\
RAM ($\beta=4$) & 0.96 & \mstd{1.85}{0.23} & \mstd{138.04/100.94}{6.73/4.36} & \mstd{667/751}{46/56} & \mstd{0.185/0.249}{0.004/0.005} \\
RAM ($\beta=5$) & 0.96 & \mstd{2.05}{0.09} & \mstd{136.78/97.16}{5.35/2.14} & \mstd{662/721}{39/21} & \mstd{0.192/0.252}{0.005/0.003} \\
RAM ($\beta=6$) & 0.95 & -- & -- & -- & -- \\
\midrule
\methodname-MMD ($\beta=4$) & 0.96 & \mstd{2.50}{0.19} & \mstd{\textcolor{blue}{124.87}/95.37}{0.83/1.60} & \mstd{\textcolor{blue}{597}/708}{5/16} & \mstd{0.197/\textcolor{blue}{0.261}}{0.006/0.008} \\
\methodname-FD ($\beta=4$) & 0.96 & \mstd{2.70}{0.34} & \mstd{\textcolor{blue}{121.46}/\textcolor{blue}{94.56}}{4.52/1.20} & \mstd{\textcolor{blue}{585}/705}{20/13} & \mstd{\textcolor{blue}{0.213}/0.255}{0.004/0.008} \\
\methodname-DS ($\beta=5$) & 0.96 & \mstd{2.08}{0.22} & \mstd{136.13/99.23}{4.59/3.20} & \mstd{\sout{653}/\sout{743}}{\sout{37}/\sout{27}} & \mstd{\textcolor{blue}{0.224}/\textcolor{blue}{0.281}}{0.008/0.003} \\
\bottomrule
\end{tabular}
}
\end{table*}

\begin{figure*}[t]
\centering
\includegraphics[width=\linewidth]{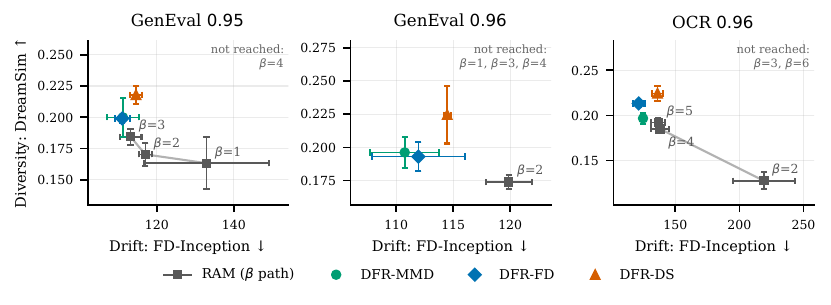}
\vspace{-1em}
\caption{Drift versus diversity at matched reward. Tuning $\beta$ moves RAM along a path (grey) that ends where a
larger $\beta$ no longer reaches the threshold. Every \methodname{} instance lies above or to the left of the end of
the path.}
\label{fig:beta-tradeoff}
\vspace{-1em}
\end{figure*}

\looseness=-1
\textbf{\methodname moves beyond the $\beta$ trade-off.} Every \methodname instance lies beyond RAM's $\beta$ path, with less drift and/or more diversity (Figure~\ref{fig:beta-tradeoff}). Further, \methodname also wins or ties against all RAM configurations that reaches the same threshold on every drift and diversity metric (Tables~\ref{tab:comparison} and~\ref{tab:beta-sweep}). The gap is widest at the highest threshold, GenEval 0.96: among RAM configurations only $\beta=$ 2 reaches it, whereas all three \methodname instances do, with lower drift (FID 111--114 vs.\ 120) and higher diversity (0.193--0.225 vs.\ 0.174). On OCR, \methodname-MMD and \methodname-FD drift the least, and \methodname-DS is the most diverse.

\textbf{The instances behave as their rewards predict.} The DreamSim reward is purely repulsive, whereas MMD and FD rewards also pull samples toward $p_\text{ref}$ (Section~\ref{subsec:special-cases-of-ours}). Accordingly, on every task and threshold, \methodname-DS attains the highest diversity of the three instances but the smallest reduction in drift, while \methodname-MMD and \methodname-FD reduce drift the most. On OCR, for example, \methodname-DS matches RAM ($\beta=5$) in FID (136 vs.\ 137), while raising diversity by 17\%, whereas \methodname-MMD and \methodname-FD lower FID to 121--125 despite running at the smaller $\beta=4$. The attractive term also acts as additional regularization: on OCR, \methodname-DS trains as fast as RAM at the same $\beta$, whereas \methodname-MMD and \methodname-FD require a smaller $\beta$ to reach the threshold and still take longer.

\textbf{\methodname-MMD matches \methodname-FD without an encoder.} \methodname-MMD computes its reward on SD3.5-M's own VAE latents with eight reference samples per prompt, whereas \methodname-FD relies on three pretrained encoders, precomputed reference statistics, and a per-encoder calibration. Despite this, \methodname-MMD ties \methodname-FD on every drift metric at every threshold, matches its diversity on GenEval, and reaches the thresholds at least as fast, although it is less diverse on OCR (0.197 vs.\ 0.213).

\textbf{\methodname adds modest cost.} Distributional rewards add only 3--9\% to the time per training step. Reaching a threshold takes longer than RAM at the same $\beta$, but the cost is comparable to that of reducing drift by raising $\beta$. Image quality and prompt generalization are also preserved:\ ImageReward and PickScore remain comparable to those of the RAM configurations that reach the same threshold.

\section{Related Work}
\label{sec:related-work}\looseness=-1
\textbf{Regularization for reward fine-tuning.}
Most reward fine-tuning methods for flow-based models use the velocity anchor in Eq.~(\ref{eq:velocity_anchor}), which constrains the sampling dynamics rather than $p_\theta$ itself (Section~\ref{sec:distributional_regularization}). DDRL~\citep{ye2025dataregularizedreinforcementlearningdiffusion} instead regularizes toward an off-policy data distribution using the forward-KL divergence, implemented through an ELBO objective on real data or pretrained-model samples. Flow Density Control~\citep{desanti2025fdc} also optimizes
functionals of the generated distribution, using gradients of
the effective reward.
\methodname instead uses the effective reward only as a scalar signal in policy-gradient
updates, without differentiating through it, which makes it directly applicable
to non-differentiable rewards, and we provide extensive
empirical evidence of scalability through systematic experiments
on SD3.5-M.

\textbf{Diversity and distributional rewards.}\looseness=-1
Objectives defined as functionals of a distribution have
been studied in reinforcement learning~\citep{hazan2019maxent,zhang2020variational}.
Several works augment the target reward with diversity or distributional signals.
GARDO~\citep{gardo2025} increases the rewards of high-quality samples far from their nearest neighbors, while DiverseGRPO~\citep{diversegrpo2025} assigns bonuses based on semantic-cluster sizes; neither signal is derived from an explicit distributional objective.
DRIFT~\citep{drift2026} rewards each sample by its mean squared DreamSim distance to the other samples in its group, which coincides with the DreamSim instance of \methodname (see Appendix~\ref{app:dreamsim-reward}).

Other methods match a prescribed target distribution. LOO-FID~\citep{loofid2026} and \citet{miao2024diverse} reward samples by their leave-one-out change in, respectively, the FD to real-image statistics and the MMD to curated reference images; these are finite-difference estimates of the corresponding first variations (Appendix~\ref{app:fd-loo}). kCGM~\citep{kcgm2026} minimizes MMD to a target feature distribution using an estimator of the form in Eq.~(\ref{eq:functional_score_identity}). DistWise~\citep{distwise2026} uses the FID between a large generated set and real data as a distribution-level reward, while DRAGON~\citep{bai2025dragon} contrasts positive and negative generation sets under distribution-level rewards.
In these methods, matching a real or curated target distribution---typically in a pretrained feature space---is itself the objective.
By contrast, \methodname derives per-sample rewards as first variations of convex functionals of $p_\theta$, yielding the explicit objective in Eq.~(\ref{eq:general_distributional_regularization}), where deviation from the pretrained model distribution regularizes target-reward maximization alongside the KL anchor. Consequently, \methodname requires neither an image dataset nor, for its MMD instance, a pretrained encoder.

\section{Conclusion}\looseness=-1
In this work, we introduce Distributional Functional Rewards (\methodname), a general-purpose framework for reward fine-tuning that complements KL regularization with objectives defined directly on the model's output distribution. By deriving per-sample rewards from the first variations of distributional functionals, \methodname provides explicit control over diversity and distributional drift. Our findings clearly demonstrate that KL regularization has a limited effective range: beyond it, stronger regularization slows optimization, limits reward improvement, and/or destabilizes training. On two distinct tasks: compositional generation and visual text rendering, we find that \methodname consistently improves the reward--distribution trade-off beyond what is attainable by tuning the KL coefficient alone. At matched task reward, every evaluated \methodname variant improves diversity, reduces drift, or both, while matching or improving all eligible distributional metrics relative to KL-only RAM baselines that reach the same reward threshold. These gains require only a simple reward augmentation, adding 3--9\% overhead per training step, while not relying on a pretrained encoder. Overall, \methodname enables reward fine-tuning to achieve high task rewards, while better preserving diversity and remaining closer to the pretrained model distribution than KL regularization alone. 

\looseness=-1
\textbf{Limitations and future work.} Although presented as a general-purpose reward fine-tuning framework, our experiments are focused on validation in image settings;\ extending \methodname to other modalities such as video and molecular structures is a natural next step to more effectively evaluate the generalization of the framework. Further, since the distributional reward regularizes toward $p_\text{ref}$, it preserves the pretrained model's biases along with its diversity, so its coefficient $\lambda$ must be tuned jointly with $\beta$. Balancing \methodname{}'s repulsive and attractive terms more carefully is a promising direction for improving this trade-off. Our characterization of the optimum assumes exact optimization of the population objective, which the ELBO surrogate and the finite-sample estimates of the distributional rewards approximate in practice.

\subsection*{Acknowledgments}
We thank Yeonwoo Cha for running baseline experiments and
for early discussions, and Jaemoo Choi and Hyungjin Chung
for early discussions. 

\newpage\clearpage
\bibliography{dfr_conference}
\bibliographystyle{dfr_conference}

\newpage\clearpage
\appendix
\section{Extended Related Work}
\label{app:related-work}

\textbf{Differentiable rewards.}
When the reward is differentiable, its gradient can be propagated directly through the generative process.
DRaFT~\citep{clark2024draft} backpropagates through the sampling chain, optionally truncating the gradient to the final steps, whereas ReFL~\citep{xu2023imagereward} backpropagates through a single late denoising step from which the clean sample is predicted directly.
Adjoint Matching~\citep{domingo-enrich2025adjoint} formulates fine-tuning as stochastic optimal control under a memoryless noise schedule, whose optimum solves Eq.~(\ref{eq:maxentrl}), and regresses the velocity toward a target defined by an adjoint state that propagates the reward gradient backward along the trajectory.
\citet{shin2026efficientadjointmatchingfinetuning} reduce its cost using a linear base drift, enabling few-step deterministic sampling and a closed-form adjoint.
VGG-Flow~\citep{liu2025vggflow} matches the difference between fine-tuned and pretrained velocities to the gradient of a learned value function, while $\nabla$-GFlowNet~\citep{liu2025resnabladb} derives gradient-informed GFlowNet objectives; both seek to preserve the pretrained prior and sample diversity.
For broader unifying perspectives on reward fine-tuning, we refer readers to \citet{uehara2024understandingreinforcementlearningbasedfinetuning,choi2026rethinking,lee2026rewardscorematchingunifying,xu2026designingreinforcementlearningdiffusion}.

\textbf{Non-differentiable rewards.}
Trajectory-based methods formulate denoising as a multi-step MDP and apply policy gradients to per-step transition likelihoods~\citep{black2023ddpo,fan2023dpok,liu2025flowgrpo,xue2025dancegrpo}.
Because these likelihoods require stochastic transitions, the sampling ODE is converted to an SDE with the same marginals.
\citet{li2025mixgrpo,he2025tempflowgrpo} improve efficiency by restricting stochastic sampling and optimization to a subset of timesteps.

ELBO-based methods instead treat the final sample as a single action and replace its intractable log-likelihood $\log p_\theta(\vx_0)$ with the ELBO surrogate $-\mathcal L_\text{FM}(\theta)$, producing an online reward-weighted regression objective on forward-noised samples~\citep{mcallister2026flow,fan2025online,xue2025advantageweightedmatchingaligning}.
Although this gives a biased estimator of the exact policy gradient, it is empirically effective.
Methods derived from other perspectives reduce to the same objective~\citep{zheng2025diffusionnft,bergmeister2026reinforceadjointmatchingscaling}; Appendix~\ref{appendix:unifying-online-rwr} gives these reductions in a common notation.
By decoupling optimization from sampling, ELBO-based methods can substantially accelerate reward maximization relative to trajectory-based methods~\citep{choi2026rethinking}.

\looseness=-1
\textbf{Distributional discrepancies as training losses}. Several few-step distillation methods match the output distribution of a student generator to that of a pretrained teacher, through the integral reverse-KL divergence~\citep{luo2024diffinstruct,yin2024onestep,yin2024improveddistributionmatchingdistillation}, general $f$-divergences~\citep{xu2025onestepdiffusionmodelsfdivergence}, or moment matching~\citep{salimans2024multistepdistillationdiffusionmodels}.
Other works match the data distribution directly.
\citet{zhou2025inductive} train few-step generators from scratch with an MMD-based moment-matching objective, and \citet{debortoli2025distributional} learn the full conditional distribution of clean data given a noisy state with scoring rules such as the energy score, rather than regressing its mean.
FD-loss~\citep{yang2026fdloss} post-trains generators by minimizing the Fr\'echet distance to real data in a representation space.

\subsection{Measuring fidelity and diversity}

\textbf{Distribution-level fidelity.}
FID~\citep{heusel2017fid} compares generated and real images through the Fr\'echet distance between Gaussian fits to their Inception-v3 features.
Its limitations have motivated several alternatives.
\citet{fddinov2} find that Inception features align poorly with human judgments of modern generators and recommend DINOv2 features instead.
KID~\citep{binkowski2018demystifying} and CMMD~\citep{jayasumana2024cmmd} replace FD with MMD, which drops the Gaussian assumption and admits an unbiased estimator.
Precision and recall~\citep{precisionrecall2019} separate fidelity from coverage by estimating the support of each distribution with $k$-nearest-neighbor balls in feature space.

\textbf{Diversity.}
LPIPS~\citep{zhang2018lpips} and DreamSim~\citep{fu2023dreamsim} compare images through deep features calibrated to human similarity judgments.
LPIPS mainly reflects low-level differences such as color and texture, whereas DreamSim also captures layout, pose, and semantic content.
The mean pairwise perceptual distance within a set of samples is a common diversity measure in reward fine-tuning~\citep{domingo-enrich2025adjoint,liu2025resnabladb,liu2025vggflow}.
The Vendi Score~\citep{friedman2023vendiscore} instead measures the effective number of distinct samples, as the exponential of the entropy of the eigenvalues of the normalized similarity matrix.

\section{Unifying online RWR}
\label{appendix:unifying-online-rwr}
Several recent ELBO-based methods have different motivations, but
their updates reduce to the same online reward-weighted regression (RWR)
objective~\citep{lee2026rewardscorematchingunifying,kveton2026advantageflowadvantageweightedsquaresrl,choi2026rethinking}.
This appendix derives the reduction for ORW-CFM-W2, FPO, AWM, RAM, DiffusionNFT, and
AdvantageFlow, and summarizes it in Table~\ref{tab:online-rwr-unified}.

\textbf{Notation.}
We translate every method into the conventions of
Section~\ref{sec:flow_preliminaries} and Algorithm~\ref{alg:method}.
Time $t=0$ is clean data and $t=1$ is noise
\footnote{ORW-CFM-W2 and FPO denote clean data at time $1$; we substitute $t\mapsto 1-t$.}
, with
$\vx_t=(1-t)\,\vx_0+t\,\pmb\epsilon$, $\pmb\epsilon\sim\mathcal N(0,I)$, and
$\vv_\text{fwd}=\pmb\epsilon-\vx_0$
.
We write $\vv_\theta$, $\vv_\text{old}$, and $\vv_\text{ref}$ for the velocities of
the current, lagged, and pretrained models, all evaluated at $(\vx_t,t)$, and omit
conditioning.
Rollouts are drawn as $\vx_0\sim p_{\theta_\text{old}}$, with
$\theta_\text{old}=\theta$ for on-policy methods, and timesteps are sampled from a distribution $p(t)$, i.e., $t\sim p(t)$.
As in the main text, $\beta$ is the coefficient of the anchor to $\vv_\text{ref}$;
we write $\gamma$ for the coefficient of the anchor to $\vv_\text{old}$.
The reward enters through a scalar weight $f(r(\vx_0),t)$.
For most methods $f=A(\vx_0)/\alpha$, where $A$ is the group-normalized advantage and $1/\alpha$ a reward coefficient; for ORW-CFM-W2,
$f=\exp(r(\vx_0)/\alpha)$.

\textbf{Unified gradient.}
Setting aside clipping and loss transformations, which we discuss for each method,
every method below updates $\theta$ along
\begin{align}
\label{eq:app-unified-rwr-gradient}
g(\theta)
&=
\mathbb E_{t,\vx_0,\pmb\epsilon}
\Big[
2(\nabla_\theta\vv_\theta)^\top
\Big(
\underbrace{\beta\,(\vv_\theta-\vv_\text{ref})}_{\text{ref.\ anchor}}
+
\underbrace{\gamma\,(\vv_\theta-\vv_\text{old})}_{\text{old anchor}}
+
\underbrace{f(r(\vx_0),t)\,(\vv_\bullet-\vv_\text{fwd})}_{\text{RWR}}
\Big)
\Big], \\
\nonumber \vv_\bullet&\in\{\vv_\theta,\vv_\text{old}\},
\end{align}
where the sampling distribution of $\vx_0$ is held fixed when differentiating.
Eq.~(\ref{eq:app-unified-rwr-gradient}) is exactly the gradient of the regression loss
\begin{equation}
\mathbb E_{t,\vx_0,\pmb\epsilon}
\Big\|
\vv_\theta
-
\sg\!\Big[
\vv_\theta
+\beta\,(\vv_\text{ref}-\vv_\theta)
+\gamma\,(\vv_\text{old}-\vv_\theta)
+f(r(\vx_0),t)\,(\vv_\text{fwd}-\vv_\bullet)
\Big]
\Big\|_2^2,
\label{eq:app-unified-rwr-loss}
\end{equation}
since the stop-gradient leaves only $2(\nabla_\theta\vv_\theta)^\top(\vv_\theta-\sg[\cdot])$.
For $\vv_\bullet=\vv_\theta$ it is also exactly the gradient of the sum of squares
$\mathbb E\bigl[f\,\|\vv_\theta-\vv_\text{fwd}\|_2^2
+\beta\|\vv_\theta-\vv_\text{ref}\|_2^2
+\gamma\|\vv_\theta-\vv_\text{old}\|_2^2\bigr]$,
in which the weight $f$ does not depend on $\theta$.

We call two objectives equivalent when their gradients agree up to a positive
constant factor, which can be absorbed into the learning rate.
A recurring case is several methods that multiply every term of their loss by the
same positive time weight $w(t)$.
We absorb this weight into the time distribution, replacing
$\mathbb E_{t\sim p}[w(t)\,h(t)]$ by $\mathbb E_{t\sim\tilde p}[h(t)]$, where
$\tilde p(t)=p(t)\,w(t)/Z$ and $Z=\mathbb E_{t\sim p}[w(t)]$.
Since $Z$ does not depend on $\theta$, the two are equivalent.
Finally, some reductions hold exactly only at $\theta=\theta_\text{old}$, i.e., for
on-policy updates; these are marked with an asterisk in Table~\ref{tab:online-rwr-unified}.

\begin{table}[!h]
\centering
\caption{Online RWR methods as instances of the unified gradient,
Eq.~(\ref{eq:app-unified-rwr-gradient}).
Methods marked with $^*$ reduce exactly only at $\theta=\theta_\text{old}$
(Appendix~\ref{app:rwr-clipping}); all other reductions are exact identities.
Time weights are relative to $t\sim\mathcal U(0,1)$: sampling $t\sim p(t)$ is equivalent
to uniform sampling with weight $p(t)$.}
\label{tab:online-rwr-unified}
\small
\begin{tabular}{lcccccc}
\toprule
Method & $\beta$ & $\gamma$ & $f$ & $\vv_\bullet$ & Time weight & Details \\
\midrule
ORW-CFM-W2 & their $\alpha$ & $0$ & $\exp(r/\alpha)$ & $\vv_\theta$ & $1$
  & App.~\ref{app:rwr-orw} \\
FPO$^*$ & $0$ & $0$ & $A$ & $\vv_\theta$ & $\tilde w_\text{FPO}(t)$
  & App.~\ref{app:rwr-fpo} \\
AWM$^{*,\dagger}$ & their $\beta$ & $0$ & $A$ & $\vv_\theta$ & $w_\text{AWM}(t)$
  & App.~\ref{app:rwr-awm} \\
RAM$^\S$ & $1$ & $0$ & $A/\alpha$ & $\vv_\text{old}$ & $2t$
  & App.~\ref{app:rwr-ram} \\
DiffusionNFT & $0^\ddagger$ & $1$ & $\tilde A/\zeta$ & $\vv_\text{old}$ & $1$
  & App.~\ref{app:rwr-nft} \\
AdvantageFlow & their $\lambda$ & $1.1$ or $1-\tilde A$ & $\tilde A$ & $\vv_\theta$
  & $t^2$ & App.~\ref{app:rwr-nft} \\
\midrule
Algorithm~\ref{alg:method} (ours) & $\beta$ & $0$
  & $(A+\lambda\,\bar r_{\mathcal R})/\alpha$ & $\vv_\text{old}$ & $2t$
  & Sec.~\ref{sec:method_online_rwr} \\
\bottomrule
\end{tabular}

\smallskip
\begin{minipage}{\linewidth}
\footnotesize
$^\dagger$Follows its Equations 13, 15, and 16; its main experiments anchor to an EMA of
the online model instead.\\
$^\S$Follows its official implementation; its Equation 17 has $\vv_\bullet=\vv_\theta$.\\
$^\ddagger$Its official implementation adds a small reference anchor, i.e., $\beta>0$.
\end{minipage}
\end{table}

\subsection{ORW-CFM-W2}
\label{app:rwr-orw}

ORW-CFM-W2~\citep{fan2025online} minimizes a reward-weighted flow-matching loss plus an anchor to the pretrained velocity~\citep[Equation 10]{fan2025online}.
Their reward weight $\exp(\tau\,r)$ and anchor coefficient $\alpha$ correspond to our
$f=\exp(r(\vx_0)/\alpha)$ and $\beta$, and rollouts are drawn from the current model,
i.e., $\theta_\text{old}=\theta$.
Their loss is therefore exactly the sum-of-squares form below
Eq.~(\ref{eq:app-unified-rwr-loss}), with $\gamma=0$ and $\vv_\bullet=\vv_\theta$
(Table~\ref{tab:online-rwr-unified}).

\subsection{FPO}
\label{app:rwr-fpo}
For each rollout $\vx_0\sim p_{\theta_\text{old}}$ with advantage $A=A(\vx_0)$, FPO
draws $N_\text{mc}$ pairs $(t_i,\pmb\epsilon_i)$ with $t_i\sim\mathcal U(0,1)$ and
evaluates the $\pmb\epsilon$-prediction loss
$\ell_\theta(t_i,\pmb\epsilon_i)
=w_\text{FPO}(t_i)\,\|\hat{\pmb\epsilon}_\theta(\vx_{t_i},t_i)-\pmb\epsilon_i\|_2^2$,
where $w_\text{FPO}(t)>0$ is the weighting
of~\citet[Equation 14]{mcallister2026flow}.
Rewriting the model's prediction as
$\hat{\pmb\epsilon}_\theta=\vx_t+(1-t)\,\vv_\theta$ gives
$\hat{\pmb\epsilon}_\theta-\pmb\epsilon=(1-t)(\vv_\theta-\vv_\text{fwd})$, and hence
\begin{equation}
\ell_\theta(t,\pmb\epsilon)
=
\tilde w_\text{FPO}(t)\,\|\vv_\theta-\vv_\text{fwd}\|_2^2,
\qquad
\tilde w_\text{FPO}(t)=w_\text{FPO}(t)\,(1-t)^2.
\label{eq:app-fpo-loss}
\end{equation}
The ratio and the clipped objective are~\citep[Algorithm 1, lines 9--10]{mcallister2026flow}
\begin{equation}
\hat\rho_\theta
=
\exp\Bigl(
-\frac{1}{N_\text{mc}}\sum_{i=1}^{N_\text{mc}}
\bigl[\ell_\theta(t_i,\pmb\epsilon_i)-\ell_{\theta_\text{old}}(t_i,\pmb\epsilon_i)\bigr]
\Bigr),
\qquad
\mathcal L_\text{FPO}(\theta)
=
-\min\bigl(\hat\rho_\theta A,\;\operatorname{clip}(\hat\rho_\theta,1-\xi,1+\xi)\,A\bigr),
\label{eq:app-fpo-ratio}
\end{equation}
where we negate FPO's objective to obtain a loss to be minimized.

At $\theta=\theta_\text{old}$, $\hat\rho_\theta=1$ lies strictly inside
$[1-\xi,1+\xi]$, so no rollout is clipped, and
$\nabla_\theta\hat\rho_\theta=\hat\rho_\theta\nabla_\theta\log\hat\rho_\theta
=\nabla_\theta\log\hat\rho_\theta$.
The gradient of Eq.~(\ref{eq:app-fpo-ratio}) is therefore
\begin{equation}
\nabla_\theta\mathcal L_\text{FPO}(\theta_\text{old})
=
-A\,\nabla_\theta\log\hat\rho_\theta
=
\frac{1}{N_\text{mc}}\sum_{i=1}^{N_\text{mc}}
2\,\tilde w_\text{FPO}(t_i)\,(\nabla_\theta\vv_\theta)^\top A\,(\vv_\theta-\vv_\text{fwd}),
\end{equation}
with the velocities in the $i$-th term evaluated at $(\vx_{t_i},t_i)$.
Averaging over $(t_i,\pmb\epsilon_i)$ and absorbing $\tilde w_\text{FPO}(t)$ into the time
distribution, this is exactly Eq.~(\ref{eq:app-unified-rwr-gradient}) with
$\beta=\gamma=0$, $f=A$, $\vv_\bullet=\vv_\theta$, and $p(t)\propto \tilde w_\text{FPO}(t)$.
For $\theta\neq\theta_\text{old}$, the importance ratio and clipping modify this
gradient (Appendix~\ref{app:rwr-clipping}).

\subsection{AWM}
\label{app:rwr-awm}

AWM~\citep{xue2025advantageweightedmatchingaligning} uses the same importance ratio
and clipping as FPO (Appendix~\ref{app:rwr-clipping}), but with a
velocity-prediction loss and an additional anchor to the pretrained model.
The regression loss entering the ratio
is~\citep[Equation 15]{xue2025advantageweightedmatchingaligning}
\begin{equation}
\ell_\theta(t,\pmb\epsilon)
=
w_\text{AWM}(t)\,\|\vv_\theta-\vv_\text{fwd}\|_2^2,
\label{eq:app-awm-loss}
\end{equation}
averaged over $N_\text{mc}$ pairs $(t_i,\pmb\epsilon_i)$, and the anchor is
$\beta\,w_\text{AWM}(t)\,\|\vv_\theta-\vv_\text{ref}\|_2^2$, whose coefficient coincides
with our $\beta$~\citep[Equations 13 and 16]{xue2025advantageweightedmatchingaligning}.
With PPO clipping, which the paper omits from its algorithm for simplicity but its
implementation includes, the full loss is
\begin{equation}
\mathcal L_\text{AWM}(\theta)
=
-\min\bigl(\hat\rho_\theta A,\;\operatorname{clip}(\hat\rho_\theta,1-\xi,1+\xi)\,A\bigr)
+
\frac{\beta}{N_\text{mc}}\sum_{i=1}^{N_\text{mc}}
w_\text{AWM}(t_i)\,\|\vv_\theta-\vv_\text{ref}\|_2^2,
\label{eq:app-awm-full}
\end{equation}
with $\hat\rho_\theta$ as in Eq.~(\ref{eq:app-clip-loss}) and $\bar\ell_\theta$ the
average of Eq.~(\ref{eq:app-awm-loss}).

As for FPO, at $\theta=\theta_\text{old}$ we have $\hat\rho_\theta=1$, no rollout is
clipped, and the ratio term contributes
$-A\,\nabla_\theta\log\hat\rho_\theta=A\,\nabla_\theta\bar\ell_\theta$.
Hence
\begin{equation}
\nabla_\theta\mathcal L_\text{AWM}(\theta_\text{old})
=
\frac{1}{N_\text{mc}}\sum_{i=1}^{N_\text{mc}}
2\,w_\text{AWM}(t_i)\,(\nabla_\theta\vv_\theta)^\top
\bigl[A\,(\vv_\theta-\vv_\text{fwd})+\beta\,(\vv_\theta-\vv_\text{ref})\bigr],
\end{equation}
with the velocities in the $i$-th term evaluated at $(\vx_{t_i},t_i)$.
The weight $w_\text{AWM}(t)$ multiplies both terms, so absorbing it into the time
distribution, this is exactly Eq.~(\ref{eq:app-unified-rwr-gradient}) with $\gamma=0$,
$f=A$, $\vv_\bullet=\vv_\theta$, and $p(t)\propto w_\text{AWM}(t)$.
For $\theta\neq\theta_\text{old}$, see Appendix~\ref{app:rwr-clipping}.

\textbf{Caveat.}
In its public implementation\footnote{\href{https://github.com/scxue/advantage_weighted_matching/blob/e20f94074c4e794d12dc953e58ee737d0efab786/advantage_weighted_matching/docs/shared_prompt_training_noise.md?plain=1\#L43}{AWM official implementation}}, AWM replaces the loss inside the ratio,
Eq.~(\ref{eq:app-awm-loss}), by a concave transform of the squared error, which it
calls a generalized Huber loss:
\begin{equation}
\ell_\theta(t,\pmb\epsilon)
=
\frac{t\,\Bigl[\bigl(\tfrac1D\|\vv_\theta-\vv_\text{fwd}\|_2^2+\delta\bigr)^{1/2}-\delta^{1/2}\Bigr]}{1/2},
\qquad
\delta=10^{-10},
\label{eq:app-awm-ghuber}
\end{equation}
where $D$ is the data dimension. This transform is not described in the paper.

\subsection{RAM}
\label{app:rwr-ram}

RAM~\citep{bergmeister2026reinforceadjointmatchingscaling} derives its objective from stochastic optimal control (SOC), but approximates the path-cost correction term by zero;
this approximation is what reduces the objective to reward-weighted regression.
Note that without the correction, the SOC formulation is poorly justified.
Therefore, we instead analyze the objective RAM optimizes in
practice~\citep[Equation 17]{bergmeister2026reinforceadjointmatchingscaling},
\begin{equation}
\mathcal L_\text{RAM}(\theta)
=
\mathbb E_{t\sim p(t),\,\vx_0\sim p_{\theta_\text{old}},\,\pmb\epsilon}
\Bigl\|
\vv_\theta
-
\sg\bigl[\vv_\text{ref}+A\,(\vv_\text{fwd}-\vv_\theta)\bigr]
\Bigr\|_2^2,
\qquad
p(t)=2t,
\label{eq:app-ram-paper}
\end{equation}
where we write $A$ for the group-normalized advantage that RAM denotes by $r$.
With $\beta=1$, the terms $\vv_\theta$ and $-\beta\,\vv_\theta$ inside the
stop-gradient of Eq.~(\ref{eq:app-unified-rwr-loss}) cancel, so
Eq.~(\ref{eq:app-ram-paper}) is exactly Eq.~(\ref{eq:app-unified-rwr-loss}) with
$\beta=1$, $\gamma=0$, $f=A$, and $\vv_\bullet=\vv_\theta$.

In its implementation, RAM makes two changes: it scales the advantage by a reward
coefficient $1/\alpha$,
and evaluates the velocity inside the RWR term at the lagged parameters, using
$\vv_\text{old}$ in place of $\vv_\theta$.
The implemented loss,
$\mathbb E\,\bigl\|\vv_\theta-\sg\bigl[\vv_\text{ref}+\tfrac{A}{\alpha}(\vv_\text{fwd}-\vv_\text{old})\bigr]\bigr\|_2^2$,
is therefore Eq.~(\ref{eq:app-unified-rwr-loss}) with $\beta=1$, $\gamma=0$,
$f=A/\alpha$, and $\vv_\bullet=\vv_\text{old}$.
Note that this is equivalent to Algorithm~\ref{alg:method} with $\beta=1$ and $\lambda=0$.

\subsection{DiffusionNFT \& AdvantageFlow}
\label{app:rwr-nft}

\textbf{DiffusionNFT.} With clipped advantage $\tilde A=\operatorname{clip}(A,-1,1)$, the loss is
\begin{equation}
\mathcal L_\text{NFT}(\theta)
=
\mathbb E_{t,\vx_0\sim p_{\theta_\text{old}},\pmb\epsilon}
\Bigl[
\tfrac{1+\tilde A}{2}\,\bigl\|(1-\zeta)\,\vv_\text{old}+\zeta\,\vv_\theta-\vv_\text{fwd}\bigr\|_2^2
+
\tfrac{1-\tilde A}{2}\,\bigl\|(1+\zeta)\,\vv_\text{old}-\zeta\,\vv_\theta-\vv_\text{fwd}\bigr\|_2^2
\Bigr],
\label{eq:app-nft-loss}
\end{equation}
where we have rewritten their $\beta$ with $\zeta ( \in \{1, 0.1\})$~\citep[Equation 5]{zheng2025diffusionnft}.
\citet[Appendix H.2.1]{lee2026rewardscorematchingunifying} shows that
Eq.~(\ref{eq:app-nft-loss}) is exactly $\zeta^2$ times
Eq.~(\ref{eq:app-unified-rwr-loss}) with $\beta=0$, $\gamma=1$,
$f=\tilde A/\zeta$, and $\vv_\bullet=\vv_\text{old}$, up to a $\theta$-independent
constant.
Their public implementation additionally adds a small anchor to the pretrained
velocity, $\beta_\text{NFT}\|\vv_\theta-\vv_\text{ref}\|_2^2$ with default
$\beta_\text{NFT}=10^{-4}$, which is not part of the paper; it corresponds to some
$\beta>0$ in Eq.~(\ref{eq:app-unified-rwr-loss}), whose exact value depends on the
implementation's loss scaling and on the adaptive weighting discussed below\footnote{\href{https://github.com/NVlabs/DiffusionNFT/blob/cbb14f84b8312b620390dfcbe2fab69c1383b104/scripts/train_nft_sd3.py\#L966}{DiffusionNFT official implementation}}.

\textbf{AdvantageFlow}~\citep[Equation 5]{kveton2026advantageflowadvantageweightedsquaresrl}
states its objective in terms of the clean-sample prediction
$\hat\vx_0=\vx_t-t\,\vv$.
Since $\hat\vx_0-\vx_0=t\,(\vv-\vv_\text{fwd})$, every term carries a factor $t^2$,
which we absorb into the time distribution $p(t)$.
The resulting objective is equivalent to Eq.~(\ref{eq:app-unified-rwr-loss}) with
$f=\tilde A$ and $\vv_\bullet=\vv_\theta$.
AdvantageFlow sets $\gamma$ so that the total coefficient $\tilde A+\gamma+\beta$ is
positive for every sample, e.g., $\gamma=1.1$ or $\gamma=1-\tilde A$, which makes each
per-sample loss strictly convex in the
prediction~\citep[Section 3.3]{kveton2026advantageflowadvantageweightedsquaresrl}.

\textbf{Caveat.}
DiffusionNFT replaces the time weighting of the flow-matching loss by an adaptive
self-normalization, which it motivates as an alternative to manual tuning of the
weighting~\citep[page 6]{zheng2025diffusionnft}.
Each of the two squared errors in Eq.~(\ref{eq:app-nft-loss}) is divided by the
stop-gradient of its mean absolute $\vx_0$-prediction error.
Since $\hat\vx_0=\vx_t-t\,\vv$ gives $\hat\vx_0-\vx_0=t\,(\vv-\vv_\text{fwd})$, the
first and second terms acquire the weights
\begin{equation}
\omega_1
=
\frac{t}{\sg\bigl(\operatorname{mean}\bigl|(1-\zeta)\,\vv_\text{old}+\zeta\,\vv_\theta-\vv_\text{fwd}\bigr|\bigr)},
\qquad
\omega_2
=
\frac{t}{\sg\bigl(\operatorname{mean}\bigl|(1+\zeta)\,\vv_\text{old}-\zeta\,\vv_\theta-\vv_\text{fwd}\bigr|\bigr)}.
\end{equation}
At $\theta=\theta_\text{old}$, both residuals equal $\vv_\text{old}-\vv_\text{fwd}$, so
$\omega_1=\omega_2$ and the gradient is that of Eq.~(\ref{eq:app-nft-loss})
multiplied by a per-sample factor.
Away from $\theta_\text{old}$, however, $\omega_1\neq\omega_2$, and the two terms no
longer combine into Eq.~(\ref{eq:app-unified-rwr-loss}).
Differentiating the weighted terms directly, the coefficient of the RWR direction
$\vv_\text{old}-\vv_\text{fwd}$ becomes
\begin{equation}
\tilde f
=
\tfrac{\zeta}{2}\bigl[(1+\tilde A)\,\omega_1-(1-\tilde A)\,\omega_2\bigr],
\end{equation}
which is no longer proportional to $\tilde A$: a rollout with $\tilde A=0$ still
receives the weight $\zeta(\omega_1-\omega_2)/2$.
The adaptive weighting therefore changes the scale \emph{and update direction}.

\subsection{Importance Sampling ratios and clipping}
\label{app:rwr-clipping}

FPO and AWM weight each rollout's regression loss by an estimated importance ratio
and clip it as in PPO~\citep{schulman2017proximalpolicyoptimizationalgorithms}.
We describe the effect for a rollout $\vx_0$ with advantage $A$.
Let $\bar\ell_\theta$ be its regression loss averaged over $N_\text{mc}$ pairs
$(t_i,\pmb\epsilon_i)$, e.g., the average of Eq.~(\ref{eq:app-fpo-loss}), and
\begin{equation}
\hat\rho_\theta=\exp\bigl(-(\bar\ell_\theta-\bar\ell_{\theta_\text{old}})\bigr),
\qquad
\mathcal L(\theta)
=
-\min\bigl(\hat\rho_\theta A,\;\operatorname{clip}(\hat\rho_\theta,1-\xi,1+\xi)\,A\bigr).
\label{eq:app-clip-loss}
\end{equation}
Since $\bar\ell_\theta$ is only an ELBO surrogate for $-\log p_\theta(\vx_0)$, up to
a $\theta$-independent constant, $\hat\rho_\theta$ estimates the likelihood ratio
$p_\theta(\vx_0)/p_{\theta_\text{old}}(\vx_0)$ rather than equaling it.

\textbf{Clipping.}
If $A>0$ and $\hat\rho_\theta>1+\xi$, or $A<0$ and $\hat\rho_\theta<1-\xi$, the minimum
in Eq.~(\ref{eq:app-clip-loss}) selects the clipped branch, which is constant in
$\theta$, and the rollout contributes no gradient.
In the other two cases the minimum selects $\hat\rho_\theta A$.
Clipping thus stops updating a rollout once its ratio has moved past $1\pm\xi$ in the
direction favored by $A$, a trust region around $\theta_\text{old}$.

\textbf{Importance sampling ratio.}
When the rollout is not clipped, using
$\nabla_\theta\hat\rho_\theta=\hat\rho_\theta\nabla_\theta\log\hat\rho_\theta$,
\begin{equation}
\nabla_\theta\mathcal L(\theta)
=
-A\,\hat\rho_\theta\,\nabla_\theta\log\hat\rho_\theta
=
\hat\rho_\theta\,A\,\nabla_\theta\bar\ell_\theta,
\label{eq:app-clip-grad}
\end{equation}
i.e., the RWR gradient $A\,\nabla_\theta\bar\ell_\theta$ multiplied by the
per-rollout factor $\hat\rho_\theta$.
At $\theta=\theta_\text{old}$, $\hat\rho_\theta=1$ and no rollout is clipped, which is
why the reductions of FPO and AWM hold exactly there.
Away from $\theta_\text{old}$, the step is reweighted by $\hat\rho_\theta$ and
truncated by clipping, and Eq.~(\ref{eq:app-unified-rwr-gradient}) captures neither.

\textbf{Log-ratio surrogate.}
\citet{huangppoclip} show that PPO clipping is a hinge loss on $\rho-1$, and that
$\rho-1$ can be replaced by other functions that change sign at $\rho=1$, such as
$\log\rho$; PCPO~\citep{lee2025pcpo} reports that $\widehat{\log \rho}$ is a lower-variance estimator than $\hat \rho$ for $\rho$ for diffusion models.
Replacing $\hat\rho_\theta$ by $\log\hat\rho_\theta+1$ in
Eq.~(\ref{eq:app-clip-loss}) turns the unclipped gradient into
$-A\,\nabla_\theta\log\hat\rho_\theta=A\,\nabla_\theta\bar\ell_\theta$ for every
$\theta$, removing the multiplicative factor $\hat\rho_\theta$.
Clipping $\log\hat\rho_\theta+1$ to $[1-\xi,1+\xi]$ is the same as clipping
$\log\hat\rho_\theta$ to $[-\xi,\xi]$, so the trust region remains.

\section{Distributional rewards: proofs and details}
\label{app:connections-and-special-cases}
Throughout this appendix, we interchangeably use $q\equiv p_\text{ref}$ for brevity.

\subsection{Informal derivation of the first-order condition}
\label{sec:loose-derivation}

We derive Eq.~(\ref{eq:general_functional_gibbs}) formally, assuming $p^\star$ has a
density that is positive wherever $p_\text{ref}$ is, and that the functional
derivative of Eq.~(\ref{eq:functional_derivative}) exists; \citet[Proposition~2.5]{hu2021meanfield}
give a rigorous version under conditions on $\mathcal F$ and $p_\text{ref}$.

Write $\mathcal V(p)=\mathcal F(p)+\beta\,\mathcal D_\text{KL}(p\Vert p_\text{ref})$.
The functional derivative of the KL term is
\begin{equation}
\frac{\delta}{\delta p}\,\mathcal D_\text{KL}(p\Vert p_\text{ref})(\vx)
=
\log\frac{p(\vx)}{p_\text{ref}(\vx)}+1,
\label{eq:kl-functional-derivative}
\end{equation}
since, writing the KL divergence as $\int p\log(p/p_\text{ref})\,\mathrm d\vx$ and
differentiating under the integral,
\begin{align}
\frac{\mathrm d}{\mathrm d\varepsilon}
\mathcal D_\text{KL}(p+\varepsilon\eta\Vert p_\text{ref})\Big|_{\varepsilon=0}
&=
\frac{\mathrm d}{\mathrm d\varepsilon}
\int \bigl(p(\vx)+\varepsilon\eta(\vx)\bigr)
\log\frac{p(\vx)+\varepsilon\eta(\vx)}{p_\text{ref}(\vx)}\,\mathrm d\vx\,\Big|_{\varepsilon=0}
\\
&=
\int \frac{\partial}{\partial u}\Bigl[u\log\frac{u}{p_\text{ref}(\vx)}\Bigr]_{u=p(\vx)}
\cdot
\frac{\mathrm d}{\mathrm d\varepsilon}\bigl(p(\vx)+\varepsilon\eta(\vx)\bigr)
\,\mathrm d\vx
\\
&=
\int\Bigl(\log\frac{p(\vx)}{p_\text{ref}(\vx)}+1\Bigr)\,\eta(\vx)\,\mathrm d\vx.
\end{align}
If $p^\star$ minimizes $\mathcal V$, then for every zero-mass $\eta$ such that
$p^\star+\varepsilon\eta$ remains a probability measure for small $\varepsilon$ of
either sign,
\begin{equation}
0
=
\frac{\mathrm d}{\mathrm d\varepsilon}\mathcal V(p^\star+\varepsilon\eta)\Big|_{\varepsilon=0}
=
\int\left[
\frac{\delta\mathcal F}{\delta p}(p^\star,\vx)
+\beta\log\frac{p^\star(\vx)}{p_\text{ref}(\vx)}
\right]\eta(\mathrm d\vx),
\end{equation}
where the $+1$ of Eq.~(\ref{eq:kl-functional-derivative}) has dropped out because
$\eta(\mathcal X)=0$.
A function whose integral against every zero-mass $\eta$ vanishes is constant, so
\begin{equation}
\frac{\delta\mathcal F}{\delta p}(p^\star,\vx)
+\beta\log\frac{p^\star(\vx)}{p_\text{ref}(\vx)}
= c
\qquad\text{for all }\vx,
\end{equation}
and solving for $p^\star$ gives
$p^\star(\vx)\propto p_\text{ref}(\vx)\exp\bigl(-\tfrac1\beta\frac{\delta\mathcal F}{\delta p}(p^\star,\vx)\bigr)$,
with the constant $c$ absorbed into $Z(p^\star)$.
Convexity of $\mathcal F$ and strict convexity of the KL term make $\mathcal V$
strictly convex; therefore, the stationary point is the unique minimizer.

\subsection{Backward consistency.}
\label{app:backward-consistency}

By the chain rule for relative entropy,
\begin{equation}
\mathcal D_\mathrm{KL}
\bigl(\mathbb P\Vert\mathbb P_{\mathrm{ref}}\bigr)
=
\mathcal D_\mathrm{KL}
\bigl(\mathbb P_0\Vert p_{\mathrm{ref}}\bigr)
+
\mathbb E_{x_0\sim\mathbb P_0}
\left[
\mathcal D_\mathrm{KL}\!\left(
\mathbb P(\mathrm d \vx_{1:T}\mid \vx_0)
\Vert
\mathbb P_{\mathrm{ref}}(\mathrm d \vx_{1:T}\mid \vx_0)
\right)
\right].
\label{eq:path_kl_decomposition}
\end{equation}
Because \(\mathcal F\) depends only on the clean-endpoint marginal, the
conditional KL is minimized by
\begin{equation}
\mathbb P^\star(\mathrm d \vx_{1:T}\mid \vx_0)
=
\mathbb P_{\mathrm{ref}}(\mathrm d \vx_{1:T}\mid \vx_0).
\label{eq:posterior_preservation}
\end{equation}
\looseness=-1
The optimal process therefore changes the distribution of clean samples, while preserving the pretrained conditional noising process. Equivalently, its backward kernels coincide with those of the pretrained model. This posterior-preservation property was established by
\citet[Theorem~3]{uehara2024understandingreinforcementlearningbasedfinetuning} and underlies forward-process fine-tuning methods like DiffusionNFT
\citep{zheng2025diffusionnft}.

\subsection{Energy distance as an MMD}
\label{app:energy-mmd}
Following the energy score construction of
\citet{debortoli2025distributional},
we construct the kernel behind $\mathcal R_{\mathrm{MMD}}$ and show that its
squared MMD equals Eq.~(\ref{eq:energy-mmd-functional}).
Let $\psi:\mathcal X\to\mathbb R^m$ be a fixed injective map and $b\in(0,2)$, and define
\begin{equation}
\rho_b(\vx,\vy)=\|\psi(\vx)-\psi(\vy)\|_2^b.
\label{eq:energy-semimetric2}
\end{equation}
Consider distributions with a finite $b$-moment,
$\mathbb E_{\vx\sim p}\|\psi(\vx)\|_2^b<\infty$,
so that every expectation of $\rho_b$ is finite.
The function \(\rho_b\) is of negative type~\citep[Proposition~3]{sejdinovic2013distancerkhs} and induces the positive definite distance-induced kernel
\begin{equation}
k_b(\vx,\vy)
=
\frac{1}{2}
\left[
\rho_b(\vx,\vx_{\mathrm{a}})
+
\rho_b(\vy,\vx_{\mathrm{a}})
-
\rho_b(\vx,\vy)
\right],
\label{eq:energy-kernel}
\end{equation}
where \(\vx_{\mathrm{a}}\in\mathcal X\) is an arbitrary anchor~\citep[Section~4.1]{sejdinovic2013distancerkhs}.
The
squared MMD does not depend on the choice of \(\vx_{\mathrm{a}}\) and equals half of the energy distance~\citep[Theorem~22]{sejdinovic2013distancerkhs}, given as Eq.~(\ref{eq:energy-mmd-functional}).

\subsection{Convexity of the distributional functionals}
\label{app:convexity}

Section~\ref{sec:distributional_regularization} requires $\mathcal R$ to be convex
along mixtures: for $p_t=(1-t)p_0+tp_1$ and $t\in[0,1]$,
$\mathcal R(p_t)\le(1-t)\mathcal R(p_0)+t\mathcal R(p_1)$.
We verify this for the three instances of
Section~\ref{subsec:special-cases-of-ours}.

\textbf{MMD.}
Let $\nu=p-q$. By Eq.~(\ref{eq:mmd}),
$\mathcal R_{\mathrm{MMD}}(p)=Q(\nu):=\iint k_b(\vx,\vy)\,\nu(\mathrm d\vx)\,\nu(\mathrm d\vy)$,
a quadratic form in the signed measure $\nu$. Because $k_b$ is positive definite,
$Q(\nu)\ge 0$ for every signed measure $\nu$ with finite $b$-moment. Since $p\mapsto\nu$
is affine, $\nu_t=(1-t)\nu_0+t\nu_1$, and expanding the quadratic form,
\begin{equation}
Q(\nu_t)=(1-t)Q(\nu_0)+tQ(\nu_1)-t(1-t)\,Q(\nu_0-\nu_1)
\le(1-t)Q(\nu_0)+tQ(\nu_1).
\end{equation}

\textbf{DreamSim diversity.}
Let $m(p)=\mathbb E_p[z]\in\mathbb R^d$, so that
$\mathcal R_{\mathrm{DS}}(p)=\tfrac12\|m(p)\|_2^2$.
Since $m(p_t)=(1-t)m(p_0)+t\,m(p_1)$, and for any $m_0,m_1$,
\begin{align}
\tfrac12\|(1-t)m_0+t\,m_1\|_2^2
&=(1-t)\tfrac12\|m_0\|_2^2+t\,\tfrac12\|m_1\|_2^2-t(1-t)\tfrac12\|m_0-m_1\|_2^2 \\
&\le(1-t)\tfrac12\|m_0\|_2^2+t\,\tfrac12\|m_1\|_2^2,
\end{align}
$\mathcal R_{\mathrm{DS}}$ is convex.

\textbf{FD.}
Expanding the two $p$-dependent terms of Eq.~(\ref{eq:fd}) separately,
\begin{align}
\|\mu_p-\mu_q\|_2^2
&=\|\mu_p\|_2^2-2\mu_p^\top\mu_q+\|\mu_q\|_2^2,
\\
\operatorname{Tr}\Sigma_p
&=\operatorname{Tr}\bigl(\mathbb E_p[ff^\top]\bigr)-\operatorname{Tr}\bigl(\mu_p\mu_p^\top\bigr)
=\mathbb E_p\|f\|_2^2-\|\mu_p\|_2^2,
\end{align}
where the last equality uses $\operatorname{Tr}(vv^\top)=\|v\|_2^2$.
Adding the two, the $\|\mu_p\|_2^2$ terms cancel:
\begin{equation}
\|\mu_p-\mu_q\|_2^2+\operatorname{Tr}\Sigma_p
=\mathbb E_p\|f\|_2^2-2\mu_p^\top\mu_q+\|\mu_q\|_2^2.
\end{equation}
Both $\mathbb E_p\|f\|_2^2$ and $\mu_p=\mathbb E_p[f]$ are expectations under $p$
and hence linear in $p$, so this sum is affine in $p$. Substituting into
Eq.~(\ref{eq:fd}) and writing
$g(A):=\operatorname{Tr}\bigl(\Sigma_q^{1/2}A\,\Sigma_q^{1/2}\bigr)^{1/2}$,
\begin{equation}
\mathcal R_{\mathrm{FD}}(p)
=
\underbrace{\mathbb E_p\|f\|_2^2-2\mu_p^\top\mu_q+\|\mu_q\|_2^2+\operatorname{Tr}\Sigma_q}_{\text{affine in }p}
\;-\;2\,g(\Sigma_p).
\end{equation}
It remains to show that $p\mapsto g(\Sigma_p)$ is concave.

First, we show that $p\mapsto\Sigma_p$ is concave in the Loewner order, i.e.,
$\Sigma_{p_t}\succeq(1-t)\Sigma_{p_0}+t\,\Sigma_{p_1}$.
To see this, write $M_i=\mathbb E_{p_i}[ff^\top]$ and $\mu_i=\mu_{p_i}$ for
$i\in\{0,1\}$, so that $\Sigma_{p_i}=M_i-\mu_i\mu_i^\top$.
Expectations are linear in $p$, so under the mixture
\begin{equation}
\mathbb E_{p_t}[ff^\top]=(1-t)M_0+t\,M_1,
\qquad
\mu_{p_t}=(1-t)\mu_0+t\,\mu_1,
\end{equation}
and therefore
\begin{equation}
\Sigma_{p_t}
=(1-t)M_0+t\,M_1-\mu_{p_t}\mu_{p_t}^\top.
\end{equation}
The mixture of the covariances is
\begin{equation}
(1-t)\Sigma_{p_0}+t\,\Sigma_{p_1}
=(1-t)M_0+t\,M_1-(1-t)\mu_0\mu_0^\top-t\,\mu_1\mu_1^\top.
\end{equation}
The $M$ terms coincide, so the difference is
\begin{equation}
\Sigma_{p_t}-\bigl[(1-t)\Sigma_{p_0}+t\,\Sigma_{p_1}\bigr]
=(1-t)\mu_0\mu_0^\top+t\,\mu_1\mu_1^\top-\mu_{p_t}\mu_{p_t}^\top.
\label{eq:cov-defect}
\end{equation}
Expanding the last term,
\begin{equation}
\mu_{p_t}\mu_{p_t}^\top
=(1-t)^2\mu_0\mu_0^\top+t^2\mu_1\mu_1^\top
+t(1-t)\bigl(\mu_0\mu_1^\top+\mu_1\mu_0^\top\bigr),
\end{equation}
and substituting into Eq.~(\ref{eq:cov-defect}) with
$(1-t)-(1-t)^2=t(1-t)$ and $t-t^2=t(1-t)$,
\begin{align}
\Sigma_{p_t}-\bigl[(1-t)\Sigma_{p_0}+t\,\Sigma_{p_1}\bigr]
&=t(1-t)\mu_0\mu_0^\top+t(1-t)\mu_1\mu_1^\top
-t(1-t)\bigl(\mu_0\mu_1^\top+\mu_1\mu_0^\top\bigr)
\\
&=t(1-t)(\mu_0-\mu_1)(\mu_0-\mu_1)^\top.
\label{eq:cov-loewner-concave}
\end{align}
Since $t(1-t)\ge 0$ and $vv^\top$ is positive semidefinite for any vector $v$, the right-hand side is positive
semidefinite.

Second, $g$ is nondecreasing and concave on positive semidefinite matrices:
$A\mapsto\Sigma_q^{1/2}A\Sigma_q^{1/2}$ is linear and order preserving, the matrix
square root is operator monotone and operator concave~\citep[Theorem~V.1.9]{bhatia1996matrix},
and the trace is linear and order preserving. Hence
\begin{equation}
g(\Sigma_{p_t})
\ge g\bigl((1-t)\Sigma_{p_0}+t\Sigma_{p_1}\bigr)
\ge(1-t)\,g(\Sigma_{p_0})+t\,g(\Sigma_{p_1}),
\end{equation}
using monotonicity with Eq.~(\ref{eq:cov-loewner-concave}) and then concavity.
Thus $-2g(\Sigma_p)$ is convex, and $\mathcal R_{\mathrm{FD}}$ is convex.

\subsection{MMD reward}
\label{app:mmd-reward}

\MMDEffectiveReward*

We use $q \equiv p_\text{ref}$. Also, assume $p, q$ have finite $b$-moments, so
that every expectation is finite
\footnote{For experiments in this paper, latent values are bounded and thus satisfy this mild assumption.}.

Name the three terms of Eq.~(\ref{eq:energy-mmd-functional}):
\begin{align}
T_1(p)&=\iint\rho(\vx,\vy)\,p(\mathrm d\vx)\,q(\mathrm d\vy), \\
T_2(p)&=\iint\rho(\vx,\vx')\,p(\mathrm d\vx)\,p(\mathrm d\vx'), \\
T_3&=\iint\rho(\vy,\vy')\,q(\mathrm d\vy)\,q(\mathrm d\vy'),
\end{align}
so that $\mathcal R_{\mathrm{MMD}}(p)=T_1(p)-\tfrac12T_2(p)-\tfrac12T_3$, with $T_3$
independent of $p$. Let $\eta$ be a signed measure with $\eta(\mathcal X)=0$ and
replace $p$ by $p+\varepsilon\eta$.

$T_1$ is linear in $p$:
\begin{equation}
T_1(p+\varepsilon\eta)
=T_1(p)+\varepsilon\int\Bigl[\int\rho(\vx,\vy)\,q(\mathrm d\vy)\Bigr]\eta(\mathrm d\vx)
=T_1(p)+\varepsilon\int\mathbb E_{\vy\sim q}\,\rho(\vx,\vy)\;\eta(\mathrm d\vx).
\end{equation}

$T_2$ is quadratic in $p$. Expanding the product $(p+\varepsilon\eta)(p+\varepsilon\eta)$,
\begin{align}
T_2(p+\varepsilon\eta)
&=T_2(p)
+\varepsilon\iint\rho(\vx,\vx')\,\eta(\mathrm d\vx)\,p(\mathrm d\vx')
+\varepsilon\iint\rho(\vx,\vx')\,p(\mathrm d\vx)\,\eta(\mathrm d\vx')
\nonumber\\
&\qquad
+\varepsilon^2\iint\rho(\vx,\vx')\,\eta(\mathrm d\vx)\,\eta(\mathrm d\vx')
\\
&=T_2(p)
+2\varepsilon\int\Bigl[\int\rho(\vx,\vx')\,p(\mathrm d\vx')\Bigr]\eta(\mathrm d\vx)
+O(\varepsilon^2)
\\
&=T_2(p)
+2\varepsilon\int\mathbb E_{\vx'\sim p}\,\rho(\vx,\vx')\;\eta(\mathrm d\vx)
+O(\varepsilon^2),
\end{align}
where we use the symmetry $\rho(\vx,\vx')=\rho(\vx',\vx)$.

Combining, and dividing by $\varepsilon$,
\begin{equation}
\lim_{\varepsilon\to 0}
\frac{\mathcal R_{\mathrm{MMD}}(p+\varepsilon\eta)-\mathcal R_{\mathrm{MMD}}(p)}{\varepsilon}
=\int\Bigl[
\mathbb E_{\vy\sim q}\,\rho(\vx,\vy)
-\mathbb E_{\vx'\sim p}\,\rho(\vx,\vx')
\Bigr]\eta(\mathrm d\vx).
\end{equation}
Comparing with Eq.~(\ref{eq:functional_derivative}), a representative of the
functional derivative is
\begin{equation}
\frac{\delta\mathcal R_{\mathrm{MMD}}}{\delta p}(p,\vx)
=\mathbb E_{\vy\sim q}\,[\rho_b(\vx,\vy)]-\mathbb E_{\vx'\sim p}\,[\rho_b(\vx,\vx')],
\end{equation}
and its negative is Eq.~(\ref{eq:mmd-effective-reward}).

\subsection{FD Reward}
\label{app:fd-reward}

\subsubsection{Proof of Proposition~\ref{prop:fd-effective-reward}}
\label{app:fd-reward-proof}

\FDEffectiveReward*

\begin{proof}
Write $f=\phi(\vx)$ and recall $\mathcal R_{\mathrm{FD}}(p)=\FD(p,p_\text{ref})$ from Eq.~(\ref{eq:fd}), with $\mu_p,\Sigma_p$ the feature mean and covariance under $p$.
$\mathcal R_{\mathrm{FD}}$ depends on $p$ only through its feature mean
and covariance. We first differentiate the functional with respect to these two
statistics, and then apply the functional chain rule with respect to $p$.

\textbf{Derivatives with respect to the moments.}
The derivative with respect to the mean is:
\begin{equation}
\frac{\partial \mathcal R_{\mathrm{FD}}}{\partial \mu_p}
=
2(\mu_p-\mu_q).
\label{eq:app-fd-mean-derivative}
\end{equation}

For the covariance derivative, define
\begin{equation}
S
=
\Sigma_q^{1/2}\Sigma_p\Sigma_q^{1/2}.
\end{equation}
Under the positive-definiteness assumption of
Proposition~\ref{prop:fd-effective-reward}, the standard matrix differential
identity
\begin{equation}
d\,\operatorname{Tr}(S^{1/2})
=
\frac{1}{2}
\operatorname{Tr}\!\left(S^{-1/2}dS\right)
\end{equation}
gives, since
$dS=\Sigma_q^{1/2}(d\Sigma_p)\Sigma_q^{1/2}$,
\begin{align}
d\,\operatorname{Tr}(S^{1/2})
&=
\frac{1}{2}
\operatorname{Tr}\!\left(
S^{-1/2}
\Sigma_q^{1/2}
(d\Sigma_p)
\Sigma_q^{1/2}
\right)
\\
&=
\frac{1}{2}
\operatorname{Tr}\!\left(
\Sigma_q^{1/2}
S^{-1/2}
\Sigma_q^{1/2}
\,d\Sigma_p
\right).
\end{align}

The covariance-dependent part of
Eq.~(\ref{eq:fd}) is
$\operatorname{Tr}(\Sigma_p)-2\operatorname{Tr}(S^{1/2})$.
The first term contributes $I$, while the calculation above shows that the
second contributes
$-\Sigma_q^{1/2}S^{-1/2}\Sigma_q^{1/2}$.
Hence
\begin{equation}
\frac{\partial \mathcal R_{\mathrm{FD}}}{\partial \Sigma_p}
=
I
-
\Sigma_q^{1/2}
\left(
\Sigma_q^{1/2}\Sigma_p\Sigma_q^{1/2}
\right)^{-1/2}
\Sigma_q^{1/2}
=:G_p.
\label{eq:app-fd-G}
\end{equation}

\textbf{Functional derivative with respect to $p$.}
Consider a signed perturbation $\eta$ of $p$ satisfying
$\eta(\mathcal X)=0$. The induced first-order change in the mean is
\begin{equation}
d\mu_p[\eta]
=
\int f(\vx)\,\eta(d\vx).
\label{eq:app-fd-mean-variation}
\end{equation}
For the covariance, using
$\Sigma_p=\mathbb E_p[ff^\top]-\mu_p\mu_p^\top$ gives
\begin{align}
d\Sigma_p[\eta]
&=
\int f(\vx)f(\vx)^\top\,\eta(d\vx)
-
d\mu_p[\eta]\,\mu_p^\top
-
\mu_p\,d\mu_p[\eta]^\top
\\
&=
\int
\left[
f(\vx)f(\vx)^\top
-
f(\vx)\mu_p^\top
-
\mu_p f(\vx)^\top
\right]
\eta(d\vx).
\end{align}
Since $\eta(\mathcal X)=0$, we may add the $x$-independent matrix
$\mu_p\mu_p^\top$ inside the integral without changing its value. Therefore,
\begin{equation}
d\Sigma_p[\eta]
=
\int
(f(\vx)-\mu_p)(f(\vx)-\mu_p)^\top
\,\eta(d\vx).
\label{eq:app-fd-cov-variation}
\end{equation}

Applying the chain rule with
Eqs.~(\ref{eq:app-fd-mean-derivative}),
(\ref{eq:app-fd-G}),
(\ref{eq:app-fd-mean-variation}), and
(\ref{eq:app-fd-cov-variation}),
\begin{align}
d\mathcal R_{\mathrm{FD}}(p,q)[\eta]
&=
2(\mu_p-\mu_q)^\top d\mu_p[\eta]
+
\left\langle
G_p,
d\Sigma_p[\eta]
\right\rangle_F
\\
&=
\int
\Big[
2(\mu_p-\mu_q)^\top f(\vx)
+
\left\langle
G_p,
(f(\vx)-\mu_p)(f(\vx)-\mu_p)^\top
\right\rangle_F
\Big]
\eta(d\vx)
\\
&=
\int
\Big[
2(\mu_p-\mu_q)^\top f(\vx)
+
(f(\vx)-\mu_p)^\top G_p(f(\vx)-\mu_p)
\Big]
\eta(d\vx),
\end{align}
where the last equality uses
$\langle G,zz^\top\rangle_F=z^\top Gz$.

Reading off the integrand, a representative of $\delta\mathcal R_{\mathrm{FD}}/\delta p$ is the bracket, and its negative is Eq.~(\ref{eq:fd-effective-reward}).
\end{proof}

\subsubsection{Connection to Leave-One-Out FD}
\label{app:fd-loo}

We next relate the FD reward to a leave-one-out (LOO), or jackknife,
construction. Consider the empirical distribution $p_N$ of $N$ samples with
features $\{f_i\}_{i=1}^N$, mean
\begin{equation}
\mu
=
\frac{1}{N}\sum_{i=1}^N f_i,
\end{equation}
and covariance
\begin{equation}
\Sigma
=
\frac{1}{N}\sum_{i=1}^N
(f_i-\mu)(f_i-\mu)^\top.
\end{equation}
Let $p_{N\setminus i}$ denote the empirical distribution obtained after
removing sample $i$, and define the centered feature
\begin{equation}
z_i := f_i-\mu.
\end{equation}

The leave-one-out mean is
\begin{equation}
\mu_{\setminus i}
=
\frac{N\mu-f_i}{N-1}
=
\mu-\frac{z_i}{N-1},
\end{equation}
and therefore
\begin{equation}
\mu_{\setminus i}-\mu
=
-\frac{z_i}{N-1}.
\label{eq:fd-loo-mean-change}
\end{equation}

The corresponding leave-one-out covariance can be written exactly as
\begin{equation}
\Sigma_{\setminus i}
=
\frac{N}{N-1}\Sigma
-
\frac{N}{(N-1)^2}z_i z_i^\top.
\end{equation}
Hence
\begin{equation}
\Sigma_{\setminus i}-\Sigma
=
\frac{1}{N-1}\Sigma
-
\frac{N}{(N-1)^2}z_i z_i^\top.
\label{eq:fd-loo-cov-change}
\end{equation}

We now expand FD to first order around the moments $(\mu,\Sigma)$ of the
full empirical population. Viewing FD as a function of its mean and covariance,
\begin{align}
\mathcal R_{\mathrm{FD}}(p_{N\setminus i},q)
-
\mathcal R_{\mathrm{FD}}(p_N,q)
&\approx
\left(
\frac{\partial \mathcal R_{\mathrm{FD}}}{\partial \mu}
\right)^\top
(\mu_{\setminus i}-\mu)
+
\left\langle
\frac{\partial \mathcal R_{\mathrm{FD}}}{\partial \Sigma},
\Sigma_{\setminus i}-\Sigma
\right\rangle_F.
\label{eq:fd-loo-taylor-general}
\end{align}

For the mean contribution, Eqs.~(\ref{eq:app-fd-mean-derivative}) and~(\ref{eq:fd-loo-mean-change}) give
\begin{align}
\left(
\frac{\partial \mathcal R_{\mathrm{FD}}}{\partial \mu}
\right)^\top
(\mu_{\setminus i}-\mu)
&=
2(\mu-\mu_q)^\top
\left(
-\frac{z_i}{N-1}
\right)
\\
&=
-\frac{2}{N-1}
(\mu-\mu_q)^\top z_i.
\label{eq:fd-loo-mean-contribution}
\end{align}

For the covariance contribution, Eqs.~(\ref{eq:app-fd-G}) and~(\ref{eq:fd-loo-cov-change}) give
\begin{align}
\left\langle
\frac{\partial \mathcal R_{\mathrm{FD}}}{\partial \Sigma},
\Sigma_{\setminus i}-\Sigma
\right\rangle_F
&=
\left\langle
G,
\frac{1}{N-1}\Sigma
-
\frac{N}{(N-1)^2}z_i z_i^\top
\right\rangle_F
\\
&=
\frac{1}{N-1}
\langle G,\Sigma\rangle_F
-
\frac{N}{(N-1)^2}
\langle G,z_i z_i^\top\rangle_F
\\
&=
\frac{1}{N-1}
\langle G,\Sigma\rangle_F
-
\frac{N}{(N-1)^2}
z_i^\top Gz_i,
\label{eq:fd-loo-cov-contribution}
\end{align}
where we used
$\langle G,zz^\top\rangle_F=z^\top Gz$.

Combining Eqs.~(\ref{eq:fd-loo-mean-contribution}) and
(\ref{eq:fd-loo-cov-contribution}) yields
\begin{align}
\mathcal R_{\mathrm{FD}}(p_{N\setminus i},q)
-
\mathcal R_{\mathrm{FD}}(p_N,q)
&\approx
-\frac{2}{N-1}
(\mu-\mu_q)^\top z_i
\nonumber\\
&\quad+
\frac{1}{N-1}
\langle G,\Sigma\rangle_F
-
\frac{N}{(N-1)^2}
z_i^\top Gz_i.
\label{eq:fd-loo-taylor}
\end{align}

Using
\begin{equation}
\frac{N}{(N-1)^2}
=
\frac{1}{N-1}
+
O(N^{-2}),
\end{equation}
the sample-dependent part of Eq.~(\ref{eq:fd-loo-taylor}) becomes
\begin{equation}
-\frac{1}{N-1}
\left[
2(\mu-\mu_q)^\top z_i
+
z_i^\top Gz_i
\right]
+
O(N^{-2}).
\end{equation}
Since $z_i=f_i-\mu$, replacing
$2(\mu-\mu_q)^\top z_i$ by
$2(\mu-\mu_q)^\top f_i$ changes the expression only by
the sample-independent constant
$-2(\mu-\mu_q)^\top\mu$.
Therefore,
\begin{equation}
\mathcal R_{\mathrm{FD}}(p_{N\setminus i},q)
-
\mathcal R_{\mathrm{FD}}(p_N,q)
=
\frac{1}{N-1}
r_{\mathrm{FD}}(x_i;p_N)
+
C_N
+
O(N^{-2}),
\label{eq:fd-loo-connection}
\end{equation}
where $C_N$ is independent of $i$ and
\begin{equation}
r_{\mathrm{FD}}(\vx_i;p_N)
=
-
\left[
2(\mu-\mu_q)^\top f_i
+
(f_i-\mu)^\top G(f_i-\mu)
\right].
\end{equation}

Thus, to first order in $1/N$, the FD reward in
Proposition~\ref{prop:fd-effective-reward} agrees with the leave-one-out
change in FD up to a common scale and a sample-independent additive constant.
The additive constant disappears under reward centering, while the common
scale can be absorbed into the reward coefficient. Computationally, this
avoids evaluating FD separately for every leave-one-out population: the
per-sample rewards share the same population moments and matrix $G$.

LOO-FID~\citep{loofid2026} computes this leave-one-out difference directly, and differs from Proposition~\ref{prop:fd-effective-reward} in two respects.
First, its reference moments are those of real images rather than $p_\text{ref}$.
Second, it uses per-dimension variances in place of full covariances, i.e.,
$\Sigma_p=\operatorname{diag}(\sigma_{p,i}^2), \Sigma_q=\operatorname{diag}(\sigma_{q,i}^2)$.
With this specification, $G$ reduces to $\operatorname{diag}(1-\sigma_{q,i}/\sigma_{p,i})$.

\subsubsection{Numerical Stabilization}
\label{app:fd-numerics}

The main numerical difficulty arises from the covariance component of the FD
reward,
\begin{equation}
G_p
=
I
-
\Sigma_q^{1/2}
S^{-1/2}
\Sigma_q^{1/2},
\qquad
S
=
\Sigma_q^{1/2}\Sigma_p\Sigma_q^{1/2}.
\label{eq:fd-numerics-G}
\end{equation}
While the FD value itself contains $S^{1/2}$, its reward requires
$S^{-1/2}$, which can strongly amplify errors in small eigenvalues.
We therefore regularize the covariance estimates using shrinkage.

\textbf{Covariance shrinkage.}
Small covariance eigenvalues are stabilized by applying the same
shrinkage map to the generated and reference covariances:
\begin{equation}
\widetilde\Sigma
=
(1-\nu)\Sigma
+
\nu\,\tau I,
\qquad
\tau
=
\frac{\operatorname{Tr}(\Sigma)}{d},
\qquad
0\leq\nu<1.
\label{eq:fd-cov-shrinkage}
\end{equation}
Shrinkage suppresses noisy low-variance directions while preserving the
matching fixed point: if $\Sigma_p=\Sigma_q$, then
$\widetilde\Sigma_p=\widetilde\Sigma_q$.

Let
\begin{equation}
\widetilde G_p
=
I
-
\widetilde\Sigma_q^{1/2}
\left(
\widetilde\Sigma_q^{1/2}
\widetilde\Sigma_p
\widetilde\Sigma_q^{1/2}
\right)^{-1/2}
\widetilde\Sigma_q^{1/2}.
\end{equation}
Because the shrinkage target
$\tau=\operatorname{Tr}(\Sigma)/d$ itself depends on the covariance, simply
substituting $\widetilde G_p$ for $G_p$ does not give the exact gradient of FD
between the shrunk covariances. From
\begin{equation}
d\widetilde\Sigma_p
=
(1-\nu)d\Sigma_p
+
\frac{\nu}{d}
\operatorname{Tr}(d\Sigma_p)I,
\end{equation}
the corresponding corrected matrix is
\begin{equation}
G_p^{\mathrm{eff}}
=
(1-\nu)\widetilde G_p
+
\frac{\nu}{d}
\operatorname{Tr}(\widetilde G_p)I.
\label{eq:fd-shrinkage-correction}
\end{equation}
We therefore replace $G_p$ by $G_p^{\mathrm{eff}}$ in the covariance term of
the reward. This retains the exact score-function gradient of the FD objective
computed using the shrunk covariances. At the matching fixed point,
$\widetilde G_p=G_p^{\mathrm{eff}}=0$.

\textbf{Remark: FD with within-prompt covariance.}
In the FD experiments, $\mu_p$ is the feature mean over all prompts and $\Sigma_p$ is the within-prompt covariance $\mathbb E_{\vc}[\operatorname{Cov}(f\mid\vc)]$: in both the moments and the covariance term of Eq.~(\ref{eq:fd-effective-reward}), features are centered at the mean of their prompt group. The covariance over all prompts would also include the variation between prompts, which is not the diversity we want to preserve, and a separate covariance for each prompt cannot be estimated from $G=24$ samples. With this choice, the FD functional is not convex in $p$. The gradient identity of Eq.~(\ref{eq:effective_policy_gradient}) and Algorithm~\ref{alg:method} do not rely on convexity and are unaffected, but the minimizer of Eq.~(\ref{eq:general_functional_problem}) is no longer guaranteed to be unique.

\textbf{Moment estimation and diagnostics.}
Following \cite{yang2026fdloss}, we estimate the generated moments using an exponential moving average of the
raw first and second moments,
\begin{equation}
\mu_p
\leftarrow
\kappa\mu_p+(1-\kappa)\mu_{\mathrm{batch}},
\qquad
M_p
\leftarrow
\kappa M_p+(1-\kappa)M_{\mathrm{batch}},
\end{equation}
and form
$\Sigma_p=M_p-\mu_p\mu_p^\top$.
The EMA decouples the population used to estimate the covariance from the current minibatch size, providing a substantially larger effective sample population for the FD statistics.
All covariance construction and eigendecomposition operations are performed in double precision.

\textbf{Calibration of the FD coefficient.}
\label{app:fd-calibration}
We build the reference statistics once from $25\times1152$ samples of $p_\text{ref}$ at
the training sampling settings.
For each feature space $i$, we divide the FD by its noise floor $D_i$, the FD between two
disjoint halves of the reference samples, so that each space is measured in units of its
own finite-sample noise.
In a run that logs the FD reward without optimizing it, we measure each space's reward
spread relative to the advantage spread, $s_i$, after the moment EMA has turned over
(training steps 20--30).
We weight the spaces by $w_i=\bar s/s_i$ so that they contribute equally, and set
$\lambda=\phi/s_w$, where $s_w$ is the spread of the weighted combined reward, accounting
for the correlation between spaces.
The FD reward then contributes a fraction $\phi$ of the advantage spread at the
calibration point; the realized fraction varies over training.
The moment EMA uses $\kappa=0.95$ and covariance shrinkage $\nu=0.02$.

\subsection{DreamSim diversity Reward}
\label{app:dreamsim-reward}

\DreamSimEffectiveReward*

\begin{proof}
Assume $\phi_{\mathrm{DS}}(\vx)\neq 0$ for every $\vx$, so that $z(\vx)$ is defined,
and write
\begin{equation}
s(\vx,\vx')=z(\vx)^\top z(\vx')
\end{equation}
for the cosine similarity.
Since $\|z(\vx)\|_2=1$, $|s(\vx,\vx')|\le 1$, so
every integral of $s$ against finite (signed) measures below is finite; unlike
Appendix~\ref{app:mmd-reward}, no moment assumption on $p$ is needed.

By Eq.~(\ref{eq:dreamsim-diversity}),
$1-\mathrm{Div}_{\mathrm{DS}}(p)=\mathbb E_{\vx,\vx'\sim p}[s(\vx,\vx')]$, so
\begin{equation}
\mathcal R_{\mathrm{DS}}(p)
=
\frac12\iint s(\vx,\vx')\,p(\mathrm d\vx)\,p(\mathrm d\vx').
\label{eq:app-ds-quadratic}
\end{equation}
This has the form of the term $T_2$ in Appendix~\ref{app:mmd-reward}, with $\rho_b$
replaced by $s$ and an overall factor $\tfrac12$.

Let $\eta$ be a signed measure with $\eta(\mathcal X)=0$ and finite total variation
$|\eta|(\mathcal X)$, and replace $p$ by $p+\varepsilon\eta$.
Expanding the product $(p+\varepsilon\eta)(p+\varepsilon\eta)$ in
Eq.~(\ref{eq:app-ds-quadratic}),
\begin{align}
\mathcal R_{\mathrm{DS}}(p+\varepsilon\eta)
&=\mathcal R_{\mathrm{DS}}(p)
+\frac{\varepsilon}{2}\iint s(\vx,\vx')\,\eta(\mathrm d\vx)\,p(\mathrm d\vx')
+\frac{\varepsilon}{2}\iint s(\vx,\vx')\,p(\mathrm d\vx)\,\eta(\mathrm d\vx')
\nonumber\\
&\qquad
+\frac{\varepsilon^2}{2}\iint s(\vx,\vx')\,\eta(\mathrm d\vx)\,\eta(\mathrm d\vx').
\label{eq:app-ds-expansion}
\end{align}
The two first-order terms are equal: renaming the integration variables
$\vx\leftrightarrow\vx'$ in the second and using the symmetry
$s(\vx,\vx')=s(\vx',\vx)$ turns it into the first.
\begin{equation}
\mathcal R_{\mathrm{DS}}(p+\varepsilon\eta)
=\mathcal R_{\mathrm{DS}}(p)
+\varepsilon\int\Bigl[\int s(\vx,\vx')\,p(\mathrm d\vx')\Bigr]\eta(\mathrm d\vx)
+\frac{\varepsilon^2}{2}\iint s(\vx,\vx')\,\eta(\mathrm d\vx)\,\eta(\mathrm d\vx').
\end{equation}
Since $|s|\le 1$, its $\varepsilon^2$ coefficient is bounded in absolute value by
$\tfrac12\,|\eta|(\mathcal X)^2$.
Therefore,
\begin{equation}
\lim_{\varepsilon\to 0}
\frac{\mathcal R_{\mathrm{DS}}(p+\varepsilon\eta)-\mathcal R_{\mathrm{DS}}(p)}{\varepsilon}
=\int\mathbb E_{\vx'\sim p}\bigl[s(\vx,\vx')\bigr]\;\eta(\mathrm d\vx).
\end{equation}

Because $z(\vx)$ does not depend on $\vx'$, it can be taken out of the inner
expectation:
\begin{equation}
\mathbb E_{\vx'\sim p}\bigl[s(\vx,\vx')\bigr]
=\mathbb E_{\vx'\sim p}\bigl[z(\vx)^\top z(\vx')\bigr]
=z(\vx)^\top\mathbb E_{\vx'\sim p}[z(\vx')].
\end{equation}
Comparing with Eq.~(\ref{eq:functional_derivative}), a representative of the
functional derivative is
\begin{equation}
\frac{\delta\mathcal R_{\mathrm{DS}}}{\delta p}(p,\vx)
=z(\vx)^\top\mathbb E_{\vx'\sim p}[z(\vx')],
\end{equation}
and its negative is Eq.~(\ref{eq:dreamsim-effective-reward}).
\end{proof}

\textbf{Group estimate.}
For a group $\{\vx^{i}\}_{i=1}^{G}$ of i.i.d.\ samples from $p$, the population mean
$m(p)$ is estimated from the other samples in the group, excluding the sample itself
(Section~\ref{sec:method_online_rwr}):
\begin{equation}
\hat r_{\mathrm{DS}}(\vx^{i})
=
-\frac{1}{G-1}\sum_{j\neq i}z(\vx^{i})^\top z(\vx^{j}).
\label{eq:app-ds-group-estimate}
\end{equation}
Conditionally on $\vx^{i}$, each $\vx^{j}$ with $j\neq i$ is an independent draw from
$p$, so
\begin{equation}
\mathbb E\bigl[\hat r_{\mathrm{DS}}(\vx^{i})\,\big|\,\vx^{i}\bigr]
=-z(\vx^{i})^\top m(p)
=r_{\mathrm{DS}}(\vx^{i};p)
\end{equation}
exactly; including $j=i$ would instead add the sample-independent term
$-\|z(\vx^i)\|_2^2/G=-1/G$ and rescale by $(G-1)/G$.
Moreover, adding and subtracting $1$ inside the sum,
\begin{equation}
\hat r_{\mathrm{DS}}(\vx^{i})
=
\frac{1}{G-1}\sum_{j\neq i}\Bigl[1-z(\vx^{i})^\top z(\vx^{j})\Bigr]-1,
\end{equation}
so up to the constant $-1$, which group centering removes, the estimate is the mean
DreamSim cosine distance from $\vx^{i}$ to the other samples in its group.

DRIFT~\citep[Section~5.3 and Appendix~C]{drift2026} rewards each sample with the
average squared embedding distance to the other samples in its group,
$d(\vx^{i})=\frac{1}{G-1}\sum_{j\neq i}\|E(\vx^{i})-E(\vx^{j})\|_2^2$.
With the normalized DreamSim feature $E=z$,
$\|z(\vx^{i})-z(\vx^{j})\|_2^2=2-2\,z(\vx^{i})^\top z(\vx^{j})$, so
$d(\vx^{i})=2\,\hat r_{\mathrm{DS}}(\vx^{i})+2$ for every sample: after group
centering, DRIFT's reward is $2\hat r_{\mathrm{DS}}$, and the factor $2$ is absorbed
into $\lambda$.
DRIFT additionally clips this reward and standardizes its advantage separately.

\section{Experiment Details}
\label{app:experiment-details}

\subsection{Evaluation metrics}
\label{app:metrics}

\textbf{Drift.}
We measure drift against $p_\text{ref}$ rather than against a fixed real dataset,
which separates the change induced by fine-tuning from the pretrained model's own
mismatch with real data.
Computing FD per prompt and then averaging measures this change within each conditional
distribution, rather than in the marginal over prompts.
Its absolute values are far larger than standard FD values and are
meaningful only relative to one another.
We compute avg(FD-64) with Inception-v3 and DINOv2 ViT-L/14, and DreamSim diversity with DINO ViT-B/16 and OpenCLIP ViT-B/32.

\textbf{Held-out metrics.}
\methodname-FD computes its reward with ConvNeXt-V2, SigLIP2, and CLIP ViT-L/14
features, and \methodname-DS with the DINO ViT-B/16 branch of DreamSim.
A metric is not held out for an instance if it is computed with one of that instance's
encoders or a closely related one, because the instance optimizes a closely related
quantity during training.
We therefore strike through DreamSim's DINO branch and DINOv2 FD, a later model of the
same self-supervised family, for \methodname-DS.
DreamSim's OpenCLIP branch remains held out for \methodname-FD: it is a ViT-B/32 trained
on different data and fine-tuned for perceptual similarity, not the CLIP ViT-L/14 that
\methodname-FD uses.
Struckout entries are not counted as wins or losses.
\methodname-MMD compares SD3.5-M's own VAE latents and uses no pretrained feature
encoder, so every metric is held out for it.

\subsection{Evaluation protocol}
\label{app:evaluation-protocol}

\textbf{Checkpoint selection.}
During training, we evaluate the target reward every 10 steps on the task's full test
prompt set.
This curve is used to select checkpoints: when all three seeds possess a checkpoint that reaches the threshold within the compute budget, it is considered successful.
The full tables (Appendix~\ref{app:full-table}) also report configurations that do not
reach a threshold: seeds that reach it are evaluated at their first checkpoint that
reaches it and the others at their peak, taking the earliest on ties.
Steps is the mean number of training steps to the threshold over the seeds that reach
it, from which Appendix~\ref{app:speedup} derives the speedups.

\textbf{Evaluation prompts and sampling.}
From the GenEval test set (``g''; 2{,}212 prompts), DrawBench (``d''; 1{,}000), and the
OCR test set (``o''; 1{,}018), we draw a fixed random subset of 256 prompts each.
Models fine-tuned on GenEval are evaluated on g and d, and models fine-tuned on OCR on o
and d.
For every prompt, every model generates 64 images from the same 64 initial-noise seeds,
at $512\times512$ resolution with 20 sampling steps and guidance scale $4.5$; multi-reward
DiffusionNFT is sampled at its prescribed scale of $1.0$.
We run the transformer and text encoders in FP16 and decode with the VAE in FP32.

\textbf{Statistics.}
RAM and \methodname{} entries are the mean $\pm$ sample SD over three seeds; existing
methods are single checkpoints and are reported without SD.
Two configurations are tied on a metric if their means, as displayed, differ by at most
the larger of their two SDs, and otherwise the one with the better mean wins.
In the ImgRwd and PickScore columns of the full tables
(Appendix~\ref{app:full-table}), red marks a value significantly below SD3.5-M's by the
same criterion; for single-seed methods, we use the median seed SD of the RAM rows.

\subsection{Compute budget}
\label{app:compute-budget}

All runs of a task use the same compute: a fixed wall-clock budget of 288 GPU hours for GenEval and 144 for OCR.
The distributional rewards add little time per training step.
Measured on this hardware, \methodname-FD and \methodname-DS are within $3\%$ of
RAM, and \methodname-MMD takes approximately $9\%$ longer, mostly for the pairwise distances
within each group and to the reference samples.
Within the same budget, \methodname-MMD therefore completes correspondingly fewer
steps.
Speedups in Table~\ref{tab:comparison} 
include this overhead.

\subsection{Hyperparameters}
\label{app:hyperparameters}

\begin{table}[h]
\centering
\caption{Training hyperparameters, shared by RAM and every \methodname{} instance on
both tasks. Only $\beta$ and the \methodname{} terms differ between configurations.}
\label{tab:hyperparameters}
\begin{tabular}{lp{0.55\linewidth}}
\toprule
\textbf{Hyperparameter} & \textbf{Value} \\
\midrule
Base model & SD3.5-M \\
Fine-tuning & LoRA, rank 32, scaling 64 \\
Optimizer & AdamW \\
Learning rate & $3\times10^{-4}$ \\
Adam $(\beta_1,\beta_2)$ & $(0.9,\,0.99)$ \\
Weight decay & $10^{-4}$ \\
Precision & FP16 mixed precision \\
Prompts per training step & 48 \\
Group size $G$ (samples per prompt) & 24 \\
Optimizer steps per training step & 1 \\
Sampling during training & 20 steps, guidance scale 2.0, $512\times512$ \\
Loss timesteps per sample & 8, sampled with weight decreasing linearly toward low noise \\
Advantage & group-centered, divided by the reward SD over the training step \\
Advantage scale $1/\alpha$ & 100 \\
Lagged policy $\vv_\text{old}$ & EMA of the trained weights, decay 0.9 \\
Reverse-KL coefficient $\beta$ & swept (Section~\ref{subsec:experimental-setup}) \\
Validation & every 10 steps, guidance scale 4.5 \\
\bottomrule
\end{tabular}
\end{table}

\textbf{\methodname{} coefficients.}
Each instance adds its group-centered reward $r_{\mathcal R}$ to the normalized
advantage with coefficient $\lambda$ (Algorithm~\ref{alg:method}), so $\lambda$ is in
units of the task advantage.
\methodname-MMD uses the energy kernel with exponent $b=1.75$ on SD3.5-M's normalized
VAE latents, with $M=8$ reference samples per prompt from $p_\text{ref}$.
\methodname-DS uses the DINO branch of DreamSim with distances normalized to $[0,1]$.
\methodname-FD combines three feature spaces, each measured in units of its own noise
floor, so its raw scale is not interpretable; we instead set
$\lambda=\phi/s_w$, where $\phi$ is the fraction of the advantage spread that the FD
reward contributes and $s_w$ is calibrated once per task
(Appendix~\ref{app:fd-calibration}).
The tables use $\lambda=0.55$ (MMD), $\phi=0.07$ (FD), and $\lambda=3$ (DS) on GenEval,
and $\lambda=1.5$, $\phi=0.09$, and $\lambda=4$ on OCR.

\textbf{Choice of $\beta$.}
We add \methodname{} at the largest $\beta$ for which RAM reaches the highest threshold
on every seed.
\methodname-DS keeps this $\beta$: its diversity reward does not measurably slow reward
optimization, so it adds diversity at no cost in speed rather than acting as further
regularization.
\methodname-MMD and \methodname-FD also pull the policy toward $p_\text{ref}$; on OCR
they no longer reach the threshold on every seed at $\beta=5$, so we use $\beta=4$.

\subsection{Baselines}
\label{app:baselines}

We compare with ELBO-based DiffusionNFT and AWM, and with trajectory-based Flow-GRPO and
zeroth-order RSM.
For Flow-GRPO, we use the original released GenEval checkpoint (revision
\texttt{5efdfea}).
For multi-reward DiffusionNFT, we use the released checkpoint, and sample at its prescribed guidance scale of $1.0$.
We reproduce GenEval-only DiffusionNFT, AWM, and RSM with one seed each, using their
released code and default settings, with the exception that we adjust AWM's effective batch size configuration to match other methods.
We evaluate each reproduced run at its first checkpoint whose test reward reaches
$0.95$; RSM's reward jumps from $0.94$ to $0.97$ between consecutive checkpoints, so we
evaluate it at $0.94$.
Their speedups are those reported by the respective papers
(Appendix~\ref{app:speedup}).

\subsection{Full Table}
\label{app:full-table}

Tables~\ref{tab:comparison-full} and~\ref{tab:beta-sweep-full} extend Tables~\ref{tab:comparison} and~\ref{tab:beta-sweep} with the DINO branch of DreamSim, held-out preference rewards, and the number of training steps to each threshold.
They also include the configurations that do not reach a threshold on every seed, evaluated at the first checkpoint that reaches it or, failing that, at their peak (Appendix~\ref{app:evaluation-protocol}); these rows are evaluated at a lower reward and are therefore not used to color the \methodname{} entries.

\textbf{DreamSim branches.}
The DINO branch agrees with the OpenCLIP branch reported in the main text: on GenEval, \methodname-MMD and \methodname-FD win against every RAM configuration reaching the same threshold on the test prompts at both thresholds, and on OCR they tie or win.

\textbf{Preference rewards.}
No \methodname{} instance scores significantly below SD3.5-M on GenEval, whereas several existing methods do on DrawBench.
On OCR, RAM ($\beta=2$) falls significantly below SD3.5-M on both preference rewards and both prompt sets, consistent with the instability of small $\beta$ (Appendix~\ref{app:stability}); \methodname-DS and RAM ($\beta=5$) are marginally below it in PickScore on DrawBench.

\textbf{Multi-reward training.}
Multi-reward DiffusionNFT, trained on PickScore and related preference rewards in addition to GenEval, raises both preference rewards well above those of every other method, yet drifts more than GenEval-only DiffusionNFT in DINOv2 features ($1468/1702$ versus $1352/1437$) and is no more diverse ($0.102/0.106$ versus $0.109/0.143$).
Training on more rewards is therefore not in itself a remedy for distribution drift or mode collapse.

\begin{table*}[h]
\centering
\caption{All metrics for models fine-tuned on GenEval, at thresholds $0.95$ and $0.96$. Reward is the threshold for configurations that reach it and, otherwise, the mean GenEval at the evaluated checkpoints; Steps is the mean number of training steps to the threshold. Entries of the form $a/b$ report results on the GenEval test prompts ($a$) and on DrawBench ($b$), as means over three seeds with SD beneath; existing methods are single checkpoints. Greyed-out rows ($^\dagger$) do not reach the threshold on every seed. Blue and struck-through entries are as in Table~\ref{tab:comparison}; in the ImgRwd (ImageReward) and PickScore columns, red marks a value significantly below that of SD3.5-M (Appendix~\ref{app:evaluation-protocol}).}
\label{tab:comparison-full}
\resizebox{\textwidth}{!}{
\begin{tabular}{lcccccccc}
\toprule
\textbf{Method} & \textbf{Reward} & \textbf{Steps} $(\downarrow)$ & \textbf{FID} $(\downarrow)$ & \textbf{FD-DINOv2} $(\downarrow)$ & \textbf{DreamSim-DINO} $(\uparrow)$ & \textbf{DreamSim-OpenCLIP} $(\uparrow)$ & \textbf{ImgRwd} $(\uparrow)$ & \textbf{PickScore} $(\uparrow)$ \\
\midrule
\multicolumn{9}{l}{\emph{GenEval, threshold $0.95$}} \\
SD3.5-M & 0.63 & -- & -- & -- & 0.405/0.372 & 0.302/0.274 & 0.840/0.859 & 0.872/0.863 \\
\midrule
Flow-GRPO & 0.95 & -- & 114.32/125.61 & 1019/997 & 0.224/0.274 & 0.155/0.189 & 1.222/1.013 & 0.883/0.862 \\
RSM$^\S$ & 0.94 & -- & 132.82/154.29 & 1271/1310 & 0.184/0.249 & 0.142/0.183 & 1.082/0.838 & 0.869/\textcolor{red}{0.851} \\
AWM$^{*}$ & 0.95 & -- & 158.88/183.71 & 1412/1578 & 0.124/0.173 & 0.095/0.132 & 1.028/\textcolor{red}{0.573} & 0.869/\textcolor{red}{0.840} \\
DiffusionNFT & 0.96 & -- & 144.36/159.22 & 1352/1437 & 0.151/0.204 & 0.109/0.143 & 1.232/1.091 & 0.876/\textcolor{red}{0.851} \\
DiffusionNFT (multi)$^\ddagger$ & 0.94 & -- & 145.06/175.20 & 1468/1702 & 0.143/0.141 & 0.102/0.106 & \sout{1.481}/\sout{1.457} & \sout{0.931}/\sout{0.915} \\
\midrule
RAM ($\beta=1$) & 0.95 & \mstd{157}{55} & \mstd{132.93/144.77}{16.19/16.96} & \mstd{1211/1236}{116/185} & \mstd{0.225/0.276}{0.037/0.027} & \mstd{0.163/0.193}{0.021/0.021} & \mstd{1.151/1.013}{0.109/0.152} & \mstd{0.872/0.857}{0.010/0.010} \\
RAM ($\beta=2$) & 0.95 & \mstd{143}{15} & \mstd{116.99/122.78}{1.69/1.02} & \mstd{1078/1002}{34/11} & \mstd{0.239/0.301}{0.019/0.012} & \mstd{0.170/0.208}{0.009/0.006} & \mstd{1.209/1.069}{0.016/0.014} & \mstd{0.880/0.865}{0.002/0.001} \\
RAM ($\beta=3$) & 0.95 & \mstd{187}{21} & \mstd{113.19/122.06}{2.83/1.84} & \mstd{1030/991}{11/9} & \mstd{0.260/0.310}{0.009/0.011} & \mstd{0.184/0.214}{0.006/0.010} & \mstd{1.226/1.082}{0.020/0.009} & \mstd{0.884/0.866}{0.002/0.001} \\
\textcolor{gray}{RAM ($\beta=4$)$^\dagger$} & \textcolor{gray}{0.94} & \textcolor{gray}{\mstd{210}{28}} & \textcolor{gray}{\mstd{115.49/126.77}{7.07/14.99}} & \textcolor{gray}{\mstd{1059/1053}{107/168}} & \textcolor{gray}{\mstd{0.249/0.305}{0.030/0.023}} & \textcolor{gray}{\mstd{0.181/0.220}{0.015/0.012}} & \textcolor{gray}{\mstd{1.204/0.964}{0.044/0.157}} & \textcolor{gray}{\mstd{0.884/0.862}{0.003/0.007}} \\
\methodname-MMD ($\beta=2$, $\lambda=0.55$) & 0.95 & \mstd{183}{40} & \mstd{111.20/120.64}{4.10/6.17} & \mstd{\textcolor{blue}{990}/967}{27/52} & \mstd{\textcolor{blue}{0.295}/\textcolor{blue}{0.326}}{0.025/0.014} & \mstd{\textcolor{blue}{0.200}/0.224}{0.015/0.011} & \mstd{1.200/1.049}{0.042/0.043} & \mstd{0.882/0.863}{0.005/0.003} \\
\methodname-FD ($\beta=2$) & 0.95 & \mstd{223}{29} & \mstd{111.02/\textcolor{blue}{117.09}}{2.00/4.88} & \mstd{\textcolor{blue}{980}/\textcolor{blue}{932}}{20/46} & \mstd{\textcolor{blue}{0.282}/\textcolor{blue}{0.332}}{0.005/0.007} & \mstd{\textcolor{blue}{0.199}/\textcolor{blue}{0.231}}{0.003/0.006} & \mstd{1.199/1.034}{0.013/0.048} & \mstd{0.882/0.866}{0.003/0.002} \\
\methodname-DS ($\beta=2$) & 0.95 & \mstd{197}{25} & \mstd{114.60/121.86}{1.64/4.22} & \mstd{\sout{1034}/\sout{992}}{\sout{37}/\sout{82}} & \mstd{\sout{0.325}/\sout{0.344}}{\sout{0.017}/\sout{0.012}} & \mstd{\textcolor{blue}{0.218}/\textcolor{blue}{0.233}}{0.007/0.010} & \mstd{1.185/1.032}{0.018/0.041} & \mstd{0.881/0.863}{0.006/0.005} \\
\midrule
\multicolumn{9}{l}{\emph{GenEval, threshold $0.96$}} \\
SD3.5-M & 0.63 & -- & -- & -- & 0.405/0.372 & 0.302/0.274 & 0.840/0.859 & 0.872/0.863 \\
\midrule
\textcolor{gray}{RAM ($\beta=1$)$^\dagger$} & \textcolor{gray}{0.95} & \textcolor{gray}{\mstd{195}{49}} & \textcolor{gray}{\mstd{127.57/137.62}{8.73/10.49}} & \textcolor{gray}{\mstd{1171/1166}{86/108}} & \textcolor{gray}{\mstd{0.229/0.277}{0.012/0.016}} & \textcolor{gray}{\mstd{0.161/0.190}{0.007/0.008}} & \textcolor{gray}{\mstd{1.208/1.057}{0.053/0.056}} & \textcolor{gray}{\mstd{0.875/0.859}{0.008/0.005}} \\
RAM ($\beta=2$) & 0.96 & \mstd{210}{20} & \mstd{119.87/125.99}{2.03/3.89} & \mstd{1066/1023}{31/33} & \mstd{0.249/0.310}{0.004/0.005} & \mstd{0.174/0.213}{0.006/0.006} & \mstd{1.202/1.046}{0.035/0.060} & \mstd{0.884/0.865}{0.003/0.002} \\
\textcolor{gray}{RAM ($\beta=3$)$^\dagger$} & \textcolor{gray}{0.95} & \textcolor{gray}{220} & \textcolor{gray}{\mstd{113.91/121.93}{6.49/5.44}} & \textcolor{gray}{\mstd{1041/993}{49/53}} & \textcolor{gray}{\mstd{0.254/0.308}{0.006/0.006}} & \textcolor{gray}{\mstd{0.181/0.214}{0.005/0.008}} & \textcolor{gray}{\mstd{1.231/1.087}{0.019/0.008}} & \textcolor{gray}{\mstd{0.885/0.866}{0.004/0.003}} \\
\textcolor{gray}{RAM ($\beta=4$)$^\dagger$} & \textcolor{gray}{0.95} & -- & \textcolor{gray}{\mstd{116.29/129.16}{6.13/13.02}} & \textcolor{gray}{\mstd{1065/1073}{102/156}} & \textcolor{gray}{\mstd{0.247/0.305}{0.030/0.023}} & \textcolor{gray}{\mstd{0.180/0.221}{0.015/0.011}} & \textcolor{gray}{\mstd{1.201/0.955}{0.037/0.147}} & \textcolor{gray}{\mstd{0.884/0.862}{0.004/0.007}} \\
\methodname-MMD ($\beta=2$) & 0.96 & \mstd{223}{47} & \mstd{\textcolor{blue}{110.75}/121.67}{3.04/6.05} & \mstd{\textcolor{blue}{992}/977}{5/47} & \mstd{\textcolor{blue}{0.290}/0.321}{0.021/0.015} & \mstd{\textcolor{blue}{0.196}/0.221}{0.011/0.009} & \mstd{1.209/1.030}{0.016/0.025} & \mstd{0.884/0.864}{0.002/0.001} \\
\methodname-FD ($\beta=2$) & 0.96 & \mstd{260}{26} & \mstd{\textcolor{blue}{111.95}/\textcolor{blue}{120.40}}{4.11/4.05} & \mstd{\textcolor{blue}{994}/\textcolor{blue}{969}}{46/38} & \mstd{\textcolor{blue}{0.273}/0.319}{0.021/0.018} & \mstd{\textcolor{blue}{0.193}/0.223}{0.011/0.013} & \mstd{1.209/1.027}{0.005/0.032} & \mstd{0.883/0.865}{0.004/0.003} \\
\methodname-DS ($\beta=2$) & 0.96 & \mstd{240}{44} & \mstd{\textcolor{blue}{114.47}/124.11}{0.31/3.88} & \mstd{\sout{1052}/\sout{1007}}{\sout{8}/\sout{45}} & \mstd{\sout{0.332}/\sout{0.349}}{\sout{0.028}/\sout{0.019}} & \mstd{\textcolor{blue}{0.225}/\textcolor{blue}{0.240}}{0.022/0.018} & \mstd{1.170/0.984}{0.046/0.082} & \mstd{0.878/0.860}{0.004/0.004} \\
\bottomrule
\end{tabular}
}
\smallskip
\begin{minipage}{\linewidth}
\scriptsize
$^\dagger$Not every seed reaches the threshold; Steps averages the seeds that do. \\
$^\S$, $^*$, $^\ddagger$ as in Table~\ref{tab:comparison}; multi-reward DiffusionNFT is trained on PickScore and closely related preference rewards, so its held-out rewards are struck through.
\end{minipage}
\end{table*}

\begin{table*}[h]
\centering
\caption{All metrics for models fine-tuned on OCR, at threshold $0.96$. Columns and markings are as in Table~\ref{tab:comparison-full}, except that entries report results on the OCR test prompts ($a$) and on DrawBench ($b$).}
\label{tab:beta-sweep-full}
\resizebox{\textwidth}{!}{
\begin{tabular}{lcccccccc}
\toprule
\textbf{Method} & \textbf{Reward} & \textbf{Steps} $(\downarrow)$ & \textbf{FID} $(\downarrow)$ & \textbf{FD-DINOv2} $(\downarrow)$ & \textbf{DreamSim-DINO} $(\uparrow)$ & \textbf{DreamSim-OpenCLIP} $(\uparrow)$ & \textbf{ImgRwd} $(\uparrow)$ & \textbf{PickScore} $(\uparrow)$ \\
\midrule
\multicolumn{9}{l}{\emph{OCR, threshold $0.96$}} \\
SD3.5-M & 0.59 & -- & -- & -- & 0.307/0.372 & 0.258/0.274 & 1.004/0.859 & 0.863/0.863 \\
\midrule
RAM ($\beta=2$) & 0.96 & \mstd{67}{12} & \mstd{219.07/131.38}{24.04/7.71} & \mstd{1270/1040}{232/84} & \mstd{0.248/0.347}{0.016/0.010} & \mstd{0.128/0.245}{0.010/0.009} & \mstd{\textcolor{red}{0.545}/\textcolor{red}{0.712}}{0.251/0.133} & \mstd{\textcolor{red}{0.849}/\textcolor{red}{0.855}}{0.009/0.004} \\
\textcolor{gray}{RAM ($\beta=3$)$^\dagger$} & \textcolor{gray}{0.95} & \textcolor{gray}{\mstd{100}{0}} & \textcolor{gray}{\mstd{149.72/104.84}{5.66/3.13}} & \textcolor{gray}{\mstd{728/798}{40/58}} & \textcolor{gray}{\mstd{0.272/0.349}{0.019/0.011}} & \textcolor{gray}{\mstd{0.167/0.246}{0.007/0.004}} & \textcolor{gray}{\mstd{1.039/0.894}{0.054/0.048}} & \textcolor{gray}{\mstd{0.869/0.862}{0.003/0.002}} \\
RAM ($\beta=4$) & 0.96 & \mstd{123}{15} & \mstd{138.04/100.94}{6.73/4.36} & \mstd{667/751}{46/56} & \mstd{0.291/0.355}{0.010/0.008} & \mstd{0.185/0.249}{0.004/0.005} & \mstd{1.062/0.922}{0.064/0.045} & \mstd{0.867/0.862}{0.003/0.002} \\
RAM ($\beta=5$) & 0.96 & \mstd{137}{6} & \mstd{136.78/97.16}{5.35/2.14} & \mstd{662/721}{39/21} & \mstd{0.292/0.358}{0.003/0.004} & \mstd{0.192/0.252}{0.005/0.003} & \mstd{1.040/0.886}{0.043/0.030} & \mstd{0.867/\textcolor{red}{0.861}}{0.002/0.001} \\
\textcolor{gray}{RAM ($\beta=6$)$^\dagger$} & \textcolor{gray}{0.95} & \textcolor{gray}{\mstd{180}{28}} & \textcolor{gray}{\mstd{121.69/93.58}{5.38/2.03}} & \textcolor{gray}{\mstd{565/695}{31/26}} & \textcolor{gray}{\mstd{0.285/0.345}{0.006/0.011}} & \textcolor{gray}{\mstd{0.192/0.244}{0.003/0.010}} & \textcolor{gray}{\mstd{1.100/0.928}{0.031/0.023}} & \textcolor{gray}{\mstd{0.871/0.863}{0.002/0.002}} \\
\methodname-MMD ($\beta=4$) & 0.96 & \mstd{153}{12} & \mstd{\textcolor{blue}{124.87}/95.37}{0.83/1.60} & \mstd{\textcolor{blue}{597}/708}{5/16} & \mstd{0.300/0.367}{0.006/0.012} & \mstd{0.197/\textcolor{blue}{0.261}}{0.006/0.008} & \mstd{1.057/0.862}{0.005/0.017} & \mstd{0.868/\textcolor{red}{0.861}}{0.001/0.001} \\
\methodname-FD ($\beta=4$) & 0.96 & \mstd{183}{23} & \mstd{\textcolor{blue}{121.46}/\textcolor{blue}{94.56}}{4.52/1.20} & \mstd{\textcolor{blue}{585}/705}{20/13} & \mstd{0.295/0.359}{0.009/0.010} & \mstd{\textcolor{blue}{0.213}/0.255}{0.004/0.008} & \mstd{1.097/0.886}{0.062/0.057} & \mstd{0.870/0.861}{0.004/0.004} \\
\methodname-DS ($\beta=5$) & 0.96 & \mstd{143}{15} & \mstd{136.13/99.23}{4.59/3.20} & \mstd{\sout{653}/\sout{743}}{\sout{37}/\sout{27}} & \mstd{\sout{0.355}/\sout{0.399}}{\sout{0.010}/\sout{0.004}} & \mstd{\textcolor{blue}{0.224}/\textcolor{blue}{0.281}}{0.008/0.003} & \mstd{1.013/0.837}{0.035/0.024} & \mstd{0.864/\textcolor{red}{0.858}}{0.002/0.001} \\
\bottomrule
\end{tabular}
}
\smallskip
\begin{minipage}{\linewidth}
\footnotesize
$^\dagger$Not every seed reaches the threshold; Steps averages the seeds that do.
\end{minipage}
\end{table*}

\subsection{Speedups}
\label{app:speedup}
Table~\ref{tab:speedup} lists the speedups over Flow-GRPO reported by existing methods and
derives ours: we measure the time to reach the threshold relative to RAM ($\beta=1$) on
GenEval and RAM ($\beta=2$) on OCR, as training steps times the time per step on the same
hardware (Appendix~\ref{app:compute-budget}), and scale RAM's reported speedup accordingly.

\begin{table*}[h]
\centering
\caption{Speedups. Top: speedups over Flow-GRPO reported by the respective papers. Bottom: measured from our runs. Time is the number of steps to the threshold multiplied by the time per step, both relative to the RAM configuration stated in each block. GenEval speedups divide the reported speedup of RAM ($\beta=1$) by Time; OCR has no Flow-GRPO reference and reports Time only. Means with SD beneath are over seeds; ``--'' marks a threshold not reached on every seed.}
\label{tab:speedup}
\small
\begin{tabular}{llp{0.62\linewidth}}
\toprule
\textbf{Method} & \textbf{Speedup} & \textbf{Source} \\
\midrule
Flow-GRPO & 1 & reference \\
RSM & 6.25 & $5\times$ over TempFlow-GRPO~\citep{lee2026rewardscorematchingunifying}, which reports $1.25\times$ over Flow-GRPO~\citep{he2025tempflowgrpo} \\
AWM & 8 & reported by \citet{xue2025advantageweightedmatchingaligning} \\
DiffusionNFT & 24 & reported by \citet{zheng2025diffusionnft} \\
DiffusionNFT (multi) & 2.28 & reported by \citet{choi2026rethinking} \\
RAM ($\beta=1$) & 52.3 & $50\times$ fewer steps, times the ratio of time per step~\citep{bergmeister2026reinforceadjointmatchingscaling} \\
\bottomrule
\end{tabular}
\vspace{6pt}

\begin{tabular}{lcccc}
\toprule
\textbf{Method} & \textbf{Steps (per seed)} & \textbf{Time per step} & \textbf{Time} $(\downarrow)$ & \textbf{Speedup} $(\uparrow)$ \\
\midrule
\multicolumn{5}{l}{\emph{GenEval, threshold $0.95$; Time relative to RAM ($\beta=1$)}} \\
RAM ($\beta=1$) & 160/100/210 & 1.00 & 1.00 & 52.3 \\
RAM ($\beta=2$) & 160/140/130 & 1.00 & \mstd{0.91}{0.10} & \mstd{57.2}{6.1} \\
RAM ($\beta=3$) & 210/180/170 & 1.00 & \mstd{1.19}{0.13} & \mstd{43.9}{4.9} \\
RAM ($\beta=4$) & -- & 1.00 & -- & -- \\
\methodname-MMD ($\beta=2$) & 230/160/160 & 1.09 & \mstd{1.27}{0.28} & \mstd{41.1}{9.1} \\
\methodname-FD ($\beta=2$) & 240/240/190 & 1.03 & \mstd{1.47}{0.19} & \mstd{35.7}{4.6} \\
\methodname-DS ($\beta=2$) & 200/170/220 & 1.03 & \mstd{1.29}{0.17} & \mstd{40.4}{5.2} \\
\midrule
\multicolumn{5}{l}{\emph{OCR, threshold $0.96$; Time relative to RAM ($\beta=2$)}} \\
RAM ($\beta=2$) & 60/80/60 & 1.00 & 1.00 & -- \\
RAM ($\beta=3$) & -- & 1.00 & -- & -- \\
RAM ($\beta=4$) & 110/120/140 & 1.00 & \mstd{1.85}{0.23} & -- \\
RAM ($\beta=5$) & 140/130/140 & 1.00 & \mstd{2.05}{0.09} & -- \\
RAM ($\beta=6$) & -- & 1.00 & -- & -- \\
\methodname-MMD ($\beta=4$) & 160/140/160 & 1.09 & \mstd{2.50}{0.19} & -- \\
\methodname-FD ($\beta=4$) & 170/210/170 & 0.98 & \mstd{2.70}{0.34} & -- \\
\methodname-DS ($\beta=5$) & 160/130/140 & 0.97 & \mstd{2.08}{0.22} & -- \\
\bottomrule
\end{tabular}
\end{table*}

\subsection{Held-out rewards}
\label{app:heldout-rewards}
\begin{figure}[h]
\centering
\includegraphics[width=0.6\linewidth]{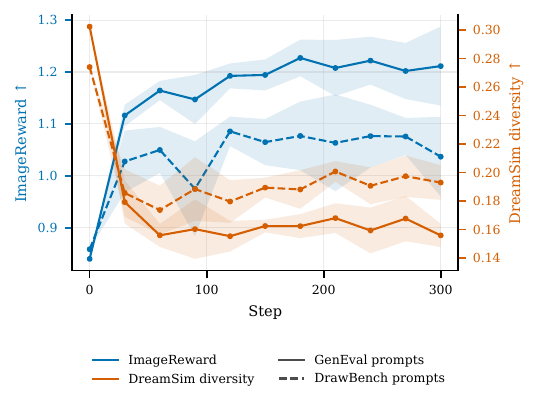}
\caption{ImageReward (left axis) and DreamSim diversity (right axis) of RAM ($\beta=1$)
during training on GenEval, on the GenEval (solid) and DrawBench (dashed) prompts,
evaluated every 30 steps; step~0 is SD3.5-M. Mean over three seeds with $\pm1$ SD
bands.}
\label{fig:heldout-rewards}
\end{figure}

Figure~\ref{fig:heldout-rewards} evaluates checkpoints of RAM ($\beta=1$) every 30 steps
with the protocol of Appendix~\ref{app:evaluation-protocol}.
Within the first 30 steps, ImageReward rises from $0.84$ to $1.12$ on the GenEval prompts
and from $0.86$ to $1.03$ on DrawBench, while DreamSim diversity falls from $0.302$ to
$0.179$ and from $0.274$ to $0.186$.
Over the rest of training, ImageReward keeps rising on the GenEval prompts, to
$1.19$--$1.23$ from step $120$ on, and stays between $0.98$ and $1.09$ on DrawBench, while
diversity stays far below SD3.5-M's, at $0.155$--$0.168$ and $0.174$--$0.201$ from step
$60$ on.
ImageReward thus rates a model higher than SD3.5-M even after its diversity on the
GenEval prompts has roughly halved: it scores each image in isolation and does not
register that the images for a prompt have become alike.
We therefore report held-out rewards (Appendix~\ref{app:full-table}) but do not use them
to measure reward hacking.

\subsection{Training stability across $\beta$}
\label{app:stability}

\begin{figure}[h]
    \centering

    \begin{subfigure}[t]{0.48\linewidth}
        \centering
        \includegraphics[width=\linewidth]{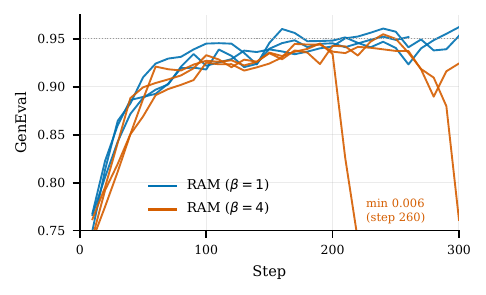}
        \caption{GenEval.}
        \label{fig:stability-geneval}
    \end{subfigure}
    \hfill
    \begin{subfigure}[t]{0.48\linewidth}
        \centering
        \includegraphics[width=\linewidth]
        {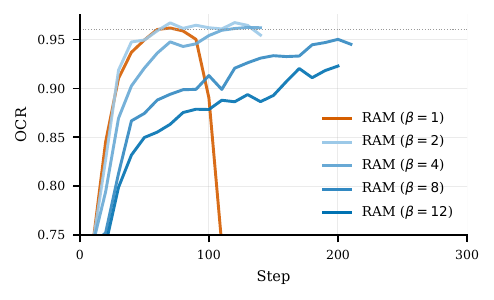}
        \caption{OCR.}
        \label{fig:stability-ocr}
    \end{subfigure}

    \caption{
        Target reward on held-out prompts during RAM training.
        \textbf{Left:} GenEval at $\beta=1$ and $\beta=4$, with one line per
        seed and the target indicated at $0.95$. At $\beta=1$, no seed
        collapses within the training budget; at $\beta=4$, every seed declines
        after peaking, and one collapses.
        \textbf{Right:} OCR for $\beta\in\{1,2,4,8,12\}$, using 1 seed for each
        coefficient and a target of $0.96$. The weakly regularized
        $\beta=1$ run collapses after reaching its peak, whereas the more
        strongly regularized runs remain stable over their observed training
        ranges, at the cost of slower reward improvement for $\beta=8$ and
        $\beta=12$.
    }
    \label{fig:stability}
\end{figure}

Training is unstable at both ends of the range of $\beta$, in different ways
(Figure~\ref{fig:stability}).
At $\beta=1$, runs disagree: the seed SDs of FD at the evaluated checkpoints are
$3$--$10\times$ those at $\beta\in\{2,3\}$ (Table~\ref{tab:comparison-full}), and the
reward dips by up to $0.037$ below its running maximum, but every dip recovers and no
seed collapses within the budget.
On OCR, the smallest $\beta=2$ likewise has FD seed SDs $4$--$6\times$ those at
$\beta\in\{4,5\}$ (Table~\ref{tab:beta-sweep-full}).
At $\beta=4$ of GenEval, training collapses: every seed peaks between steps $170$ and $240$ and then
declines.
On OCR, larger $\beta$ plateaus rather than collapses: RAM peaks at
$0.950$ at $\beta=8$ and $0.923$ at $\beta=12$.

\section{Additional Results}
\subsection{Qualitative comparison between generated images}
\label{sec:qualitative-image-comparison}
\looseness=-1
We demonstrate that \methodname produces a distribution that more closely matches that of the base SD3.5 model (see below). For the donut example, \methodname preserves the color of the glaze better than the RAM baseline. For the four-clock example, \methodname captures both the clock colors and the mix of Roman and Arabic numerals, whereas the RAM baseline collapses these variations. These examples qualitatively highlight the strengths of \methodname.
\begin{figure}[h]
\centering
\includegraphics[width=\linewidth]{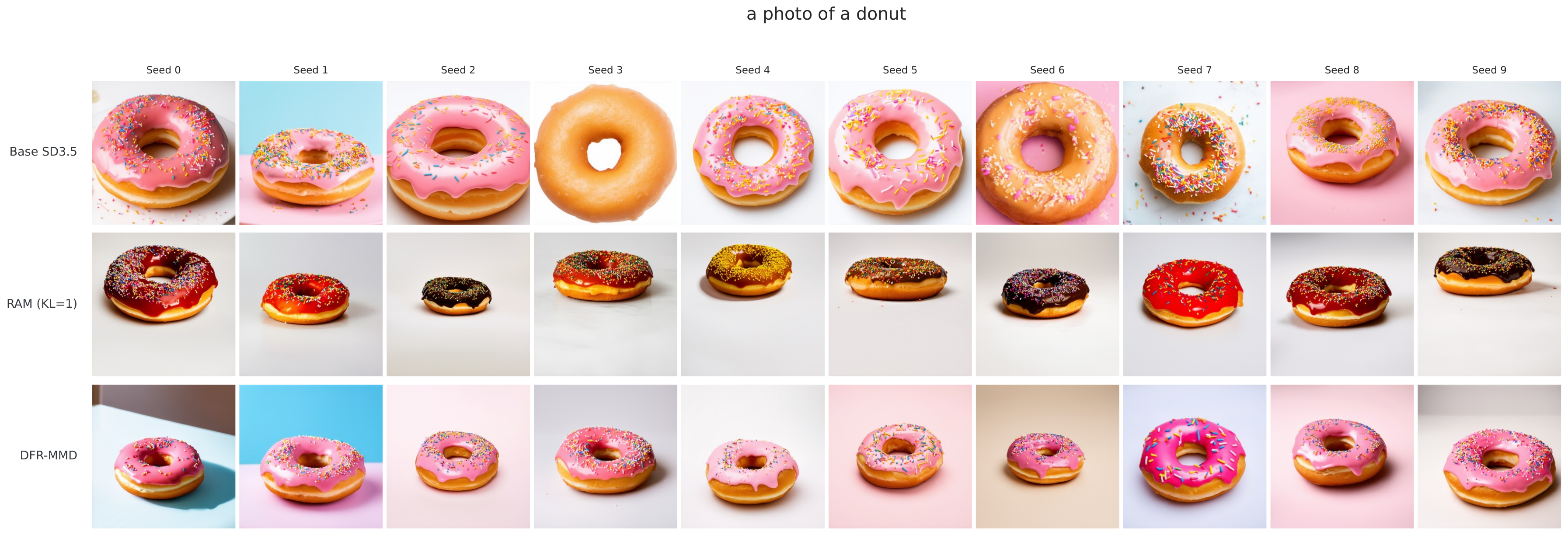}
\caption{\methodname qualitatively appears closer to the pretrained distribution;\ the colour of the glazing on the donut is preserved in \methodname, while this does not hold for the RAM baseline.}
\label{fig:donut-model-image-samples}
\end{figure}

\begin{figure}[h]
\centering
\includegraphics[width=\linewidth]{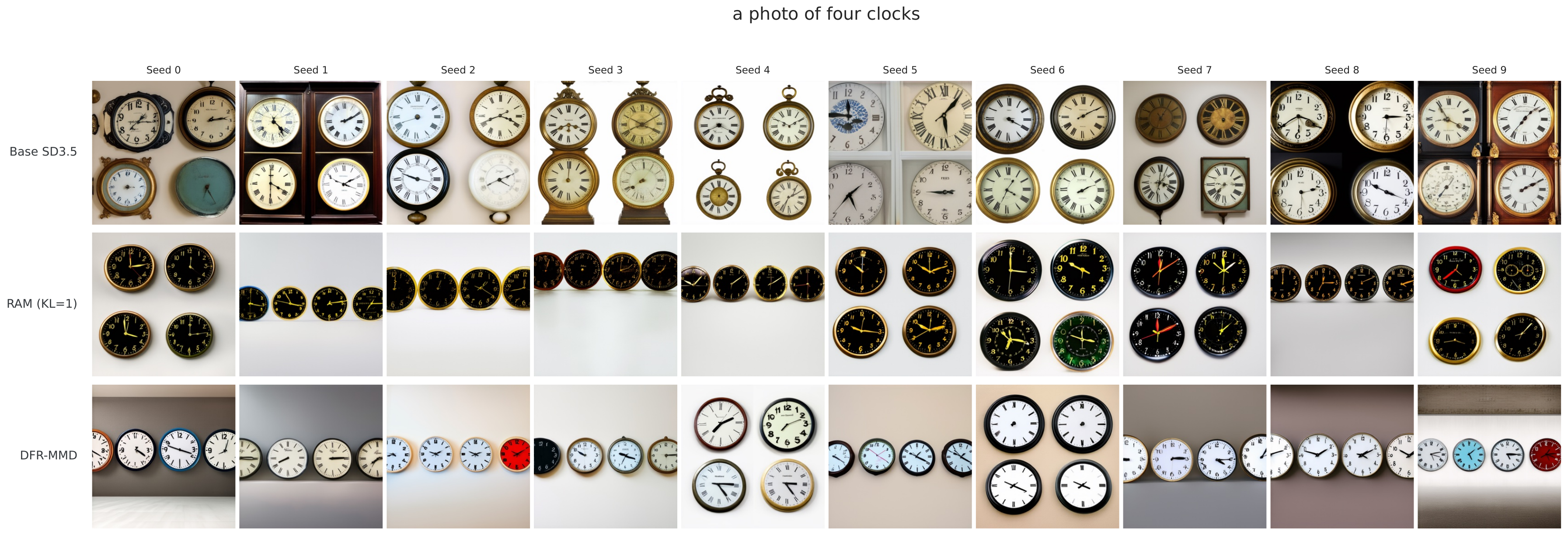}
\caption{\methodname captures both the clock colors and the mix of Roman and Arabic numerals, preserving variations from the pretrained SD3.5 distribution, whereas the RAM baseline collapses these variations.}
\label{fig:clocks-model-image-samples}
\end{figure}

\subsection{Samples from baselines}
We include samples generated from the base SD3.5 model, RAM, and \methodname below. 
\begin{figure}[h]
\centering
\includegraphics[width=0.75\linewidth]{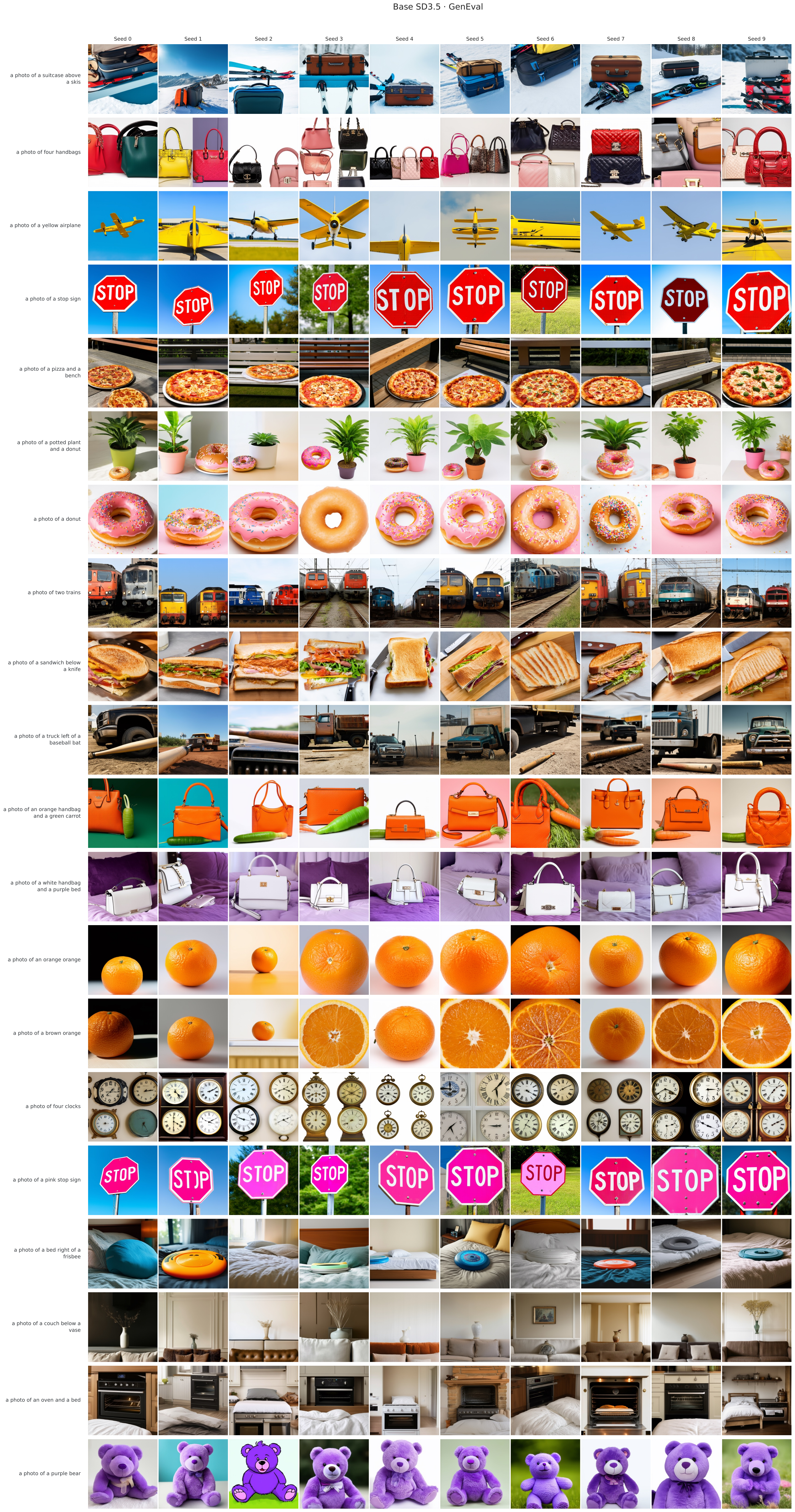}
\caption{Sample images generated for diverse prompts across various seeds using the base SD3.5 model.}
\label{fig:base-model-image-samples}
\end{figure}

\newpage\clearpage
\begin{figure}[h]
\centering
\includegraphics[width=0.75\linewidth]{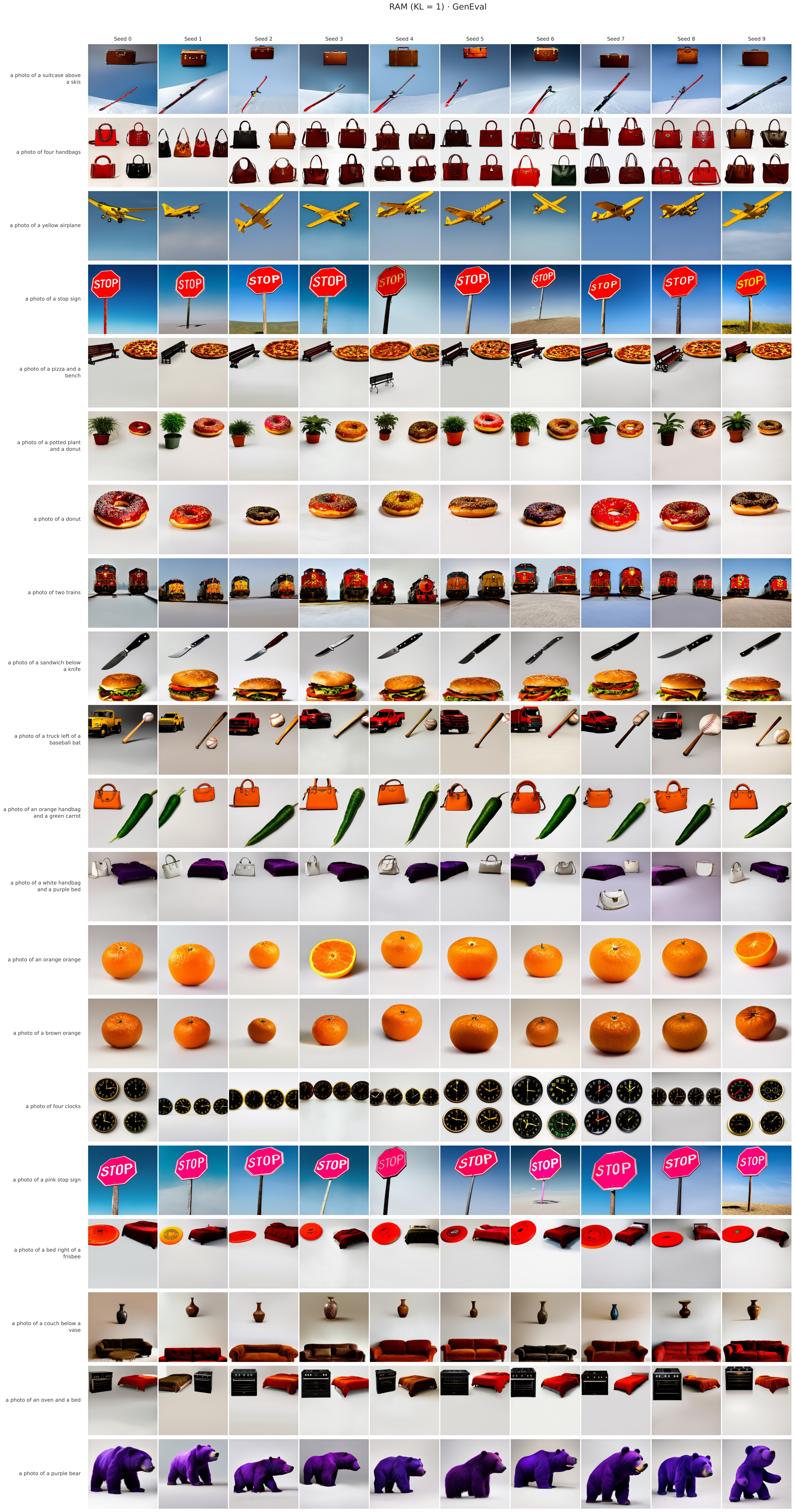}
\caption{Sample images generated for diverse prompts across various seeds using the RAM baseline.}
\label{fig:ram-model-image-samples}
\end{figure}

\newpage\clearpage
\begin{figure}[h]
\centering
\includegraphics[width=0.75\linewidth]{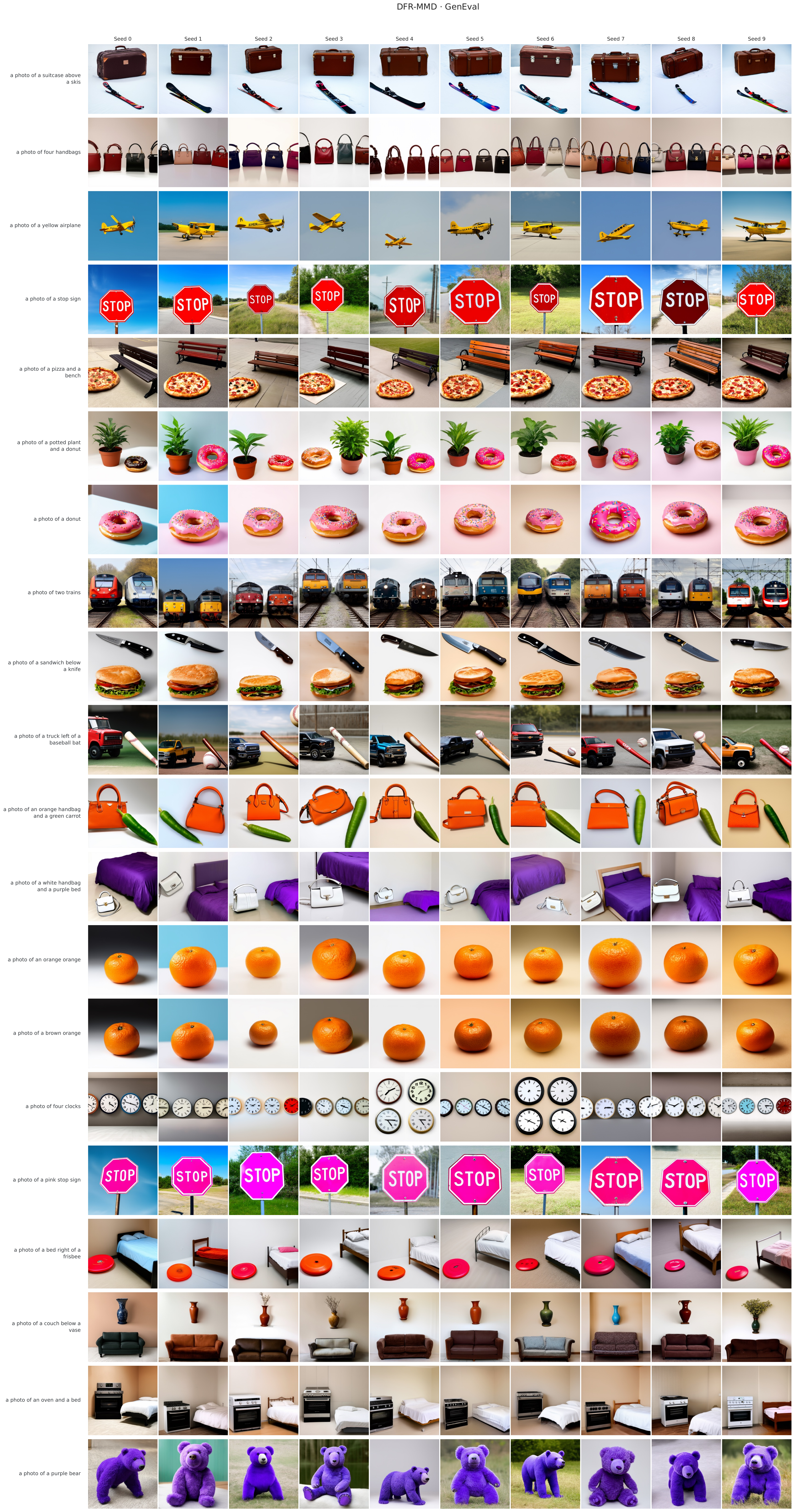}
\caption{Sample images generated for diverse prompts across various seeds using \methodname.}
\label{fig:dfr-model-image-samples}
\end{figure}

\end{document}